\documentclass{article}
\usepackage{url}
\usepackage{graphicx}
\usepackage{amsmath}
\usepackage{amssymb}
\usepackage{algorithm}
\usepackage{algorithmic}
\usepackage{color}
\definecolor{Myblue}{rgb}{0,.3,.6}
\newcommand{\emc}[1]{{\textbf{\textit{\color{Myblue}#1}}}}

\DeclareMathOperator{\erf}{erf}
\DeclareMathOperator{\erfc}{erfc}
\DeclareMathOperator{\sign}{sign}

\DeclareMathOperator{\vecc}{vec}

\newcommand{\qed}{\hfill \ensuremath{\Box}} 
\newcommand{\argmin}{\operatornamewithlimits{argmin}}
\newcommand{\argmax}{\operatornamewithlimits{argmax}}
\newtheorem{theorem}{Theorem}%[section]
\newtheorem{lemma}[theorem]{Lemma}
\newtheorem{proposition}[theorem]{Proposition}

\newenvironment{proof}[1][Proof]{\begin{trivlist}
\item[\hskip \labelsep {\bfseries #1}]}{\end{trivlist}}

\newcommand{\linesep}{\begin{center} {\color{red}\line(1,0){300}} \end{center} \vspace{-0.4in} \begin{center} {\color{red}\line(1,0){250}} \end{center}}

\DeclareFontFamily{U}{wncy}{}
\DeclareFontShape{U}{wncy}{m}{n}{<->wncyr10}{}
\DeclareSymbolFont{mcy}{U}{wncy}{m}{n}
\DeclareMathSymbol{\Sh}{\mathord}{mcy}{"58} 

\definecolor{sabz}{rgb}{0.8, 1, 0.75}
\definecolor{golbehi}{rgb}{1, 0.94, 0.94}
\definecolor{ghaazi}{rgb}{0.85,1,0.95}
\definecolor{kerem}{rgb}{1,1,0.8}

\newcommand{\mycolorbox}[2]{
\vspace{0.1in}
\setlength{\fboxrule}{1pt}
\noindent \fcolorbox{black}{#1}{\parbox{\textwidth}{#2}}
\vspace{0.1in}
}

\usepackage[breaklinks=true,colorlinks]{hyperref}
\usepackage{epigraph}
\usepackage{algorithm}
\usepackage{algorithmic}
\usepackage{multicol}

\title{A Theory of Local Matching\\ SIFT and Beyond}

\author{
Hossein Mobahi\\
CSAIL, MIT\\
Cambridge, MA\\
\color{magenta}\texttt{hmobahi@csail.mit.edu} \\
\and
Stefano Soatto\\
CS Dept., UCLA \\
Los Angeles, CA\\
\url{soatto@cs.ucla.edu}
}

\date{}

\begin{document}
\maketitle

\begin{abstract}
Why has SIFT been so successful? Why its extension, DSP-SIFT, can further improve SIFT? Is there a theory that can explain both? How can such theory benefit real applications? Can it suggest new algorithms with reduced computational complexity or new descriptors with better accuracy for matching? We construct a general theory of local descriptors for visual matching. Our theory relies on concepts in energy minimization and heat diffusion. We show that SIFT and DSP-SIFT approximate the solution the theory suggests. In particular, DSP-SIFT gives a better approximation to the theoretical solution; justifying why DSP-SIFT outperforms SIFT. Using the developed theory, we derive new descriptors that have fewer parameters and are potentially better in handling affine deformations.
\end{abstract}

\section{Introduction}

\paragraph{\bf Questions: } Why has SIFT been so successful? Why DSP-SIFT \cite{Stefano15} can further improve SIFT? Is there a theory that can explain both? How can such theory benefit real applications? Can it suggest new algorithms with reduced computational complexity or new descriptors with better accuracy for matching?

\paragraph{\bf Contributions:} We construct a general theory of local descriptors for visual matching. Our theory relies on concepts in \emc{energy minimization} and \emc{heat diffusion}. We show that SIFT and DSP-SIFT approximate the solution the theory suggests. In particular, DSP-SIFT gives a better approximation to the theoretical solution; justifying why DSP-SIFT outperforms SIFT. We derive new algorithms based on this theory. Specifically, we present a computationally efficient approximation to DSP-SIFT algorithm \cite{Stefano15} by replacing the \emc{sampling} procedure in DSP-SIFT with a closed-form approximation that does not need any sampling. This leads to a significantly faster algorithm compared to DSP-SIFT. In addition, we derive \emc{new descriptors} directly from this theory. The new descriptors have \emc{fewer parameters} as well as the potential of better handling \emc{affine deformations}, compared to SIFT and DSP-SIFT.

\section{Contributions}

Throughout this text, isotropic multivariate Gaussian kernel and periodic univariate Gaussian are denoted by $k$ and $\tilde{k}$,

\begin{equation}
k_\sigma (\boldsymbol{x}) \triangleq (2 \pi \sigma^2)^{-\dim(\boldsymbol{x})} \, e^{-\frac{\| \boldsymbol{x}\|^2 }{2 \sigma^2}} \quad,\quad \tilde{k}_\sigma (\phi) \triangleq \sum_{k=-\infty}^\infty k_\sigma (\phi+ 2 \pi k)\,. 
\end{equation}

Consider an image $f$. Given an origin-centered detected key point $\boldsymbol{x}$ with assigned scale $\sigma$ and orientation $\beta$. The continuous form of a SIFT descriptor can be expressed as \cite{Dong15,Vedaldi10},

\begin{equation}
\label{eq:sift}
h_{SIFT}(\beta , \boldsymbol{x} ) \triangleq \int_{\mathcal{X}} \tilde{k}_{\sigma_r} (\beta - \angle \nabla f (\boldsymbol{y}) ) \, k_{\sigma_d} (\boldsymbol{y} - \boldsymbol{x}) \, \| \nabla f (\boldsymbol{y}) \| \, d \boldsymbol{y} \,,
\end{equation}

where $\sigma_r$ resembles the size of each orientation bin, e.g. $\frac{2\pi}{8}$ for $8$ bins. $\sigma_d$ determines the spatial support of the descriptor as a function of $\sigma$, e.g. $\sigma_d \triangleq 3 \sigma$. 

By observing that the above descriptor is \emc{pooling} (weighted averaging) across displacement, \cite{Dong15} adds domain size pooling to this construction and suggests,

\begin{equation}
\label{eq:dsp_sift}
h_{DSP}(\beta , \boldsymbol{x} ) \triangleq \int_{\mathcal{S}} \int_{\mathcal{X}} \tilde{k}_{\sigma_r} (\beta - \angle \nabla f (\boldsymbol{y}) ) \, k_{\sigma_d} (\boldsymbol{y} - \boldsymbol{x}) \, \| \nabla f (\boldsymbol{y}) \| \, d \boldsymbol{y} k_{\sigma_s}(\sigma_d - \sigma_{d_0}) \, d \sigma_d \,,
\end{equation}

where $\mathcal{S} \triangleq \mathbb{R}$ and $\sigma_{d_0}$ is a function of key point's scale $\sigma$, e.g. $\sigma_{d_0} \triangleq 3 \sigma$.

We develop a theory for descriptor construction by returning to the origin of the problem. Specifically we formulate matching as an energy optimization problem. It is known that the resulted cost function is \emc{nonconvex} for a any realistic matching setup \cite{Stefano15}. Ideally, one would need to \emc{brute-force search} across all possible transformations to find the right match. This is obviously not practical.

Recently a theory of nonconvex optimization by \emc{heat diffusion} has been proposed \cite{MobahiEMMCVPR15,Mobahi15AAAI}. The theory offers the best (in a certain sense) tractable solution for nonconvex problems. We show that SIFT and DSP-SIFT \emc{approximate} the energy minimization solution that this theory suggests. By leveraging this connection, we present the following contributions.

The domain-size integration (\ref{eq:dsp_sift}) is approximated by \emc{numerical sampling} in \cite{Dong15}, which is slow. Instead, we propose the following two \emc{closed-form approximations} to this integral.

\begin{equation}
h_{DSP}(\beta , \boldsymbol{x} ) \approx \int_{\mathcal{X}} \tilde{k}_{\sigma_r} (\beta - \angle \nabla f (\boldsymbol{y}) ) \, k_{\sigma_d} (\boldsymbol{y} - \boldsymbol{x}) \, \| \nabla f (\boldsymbol{y}) \| \, \frac{\sigma_d}{\sqrt{\sigma_d^2 + \|\boldsymbol{x}\|^2 \sigma^2_s  }} e^{ \frac{\sigma^2_s (\|\boldsymbol{x}\|^2-\boldsymbol{x}^T \boldsymbol{y}-\sigma_d^2)^2}{2 \sigma_d^2 (\sigma_d^2+\|\boldsymbol{x}\|^2 \sigma^2_s )}} \,.
\end{equation}

\begin{eqnarray}
h_{DSP}(\beta , \boldsymbol{x} ) \approx \int_{\mathcal{X}} \tilde{k}_{\sigma_r} (\beta - \angle \nabla f (\boldsymbol{y}) ) \, k_{\sigma_d} (\boldsymbol{y} - \boldsymbol{x}) \, \| \nabla f (\boldsymbol{y}) \| \, \frac{\sigma_d^2 - 
   \sigma_s \boldsymbol{x}^T \boldsymbol{y}}{(\sigma_d^2 + \sigma_s^2 \| \boldsymbol{x}\|^2)^\frac{3}{2}} e^{-\frac{\sigma_d^2 \| \boldsymbol{x} + \boldsymbol{y} \|^2 + \sigma_s^2 (\boldsymbol{x}^T \boldsymbol{y}^\perp)^2}{
  2 \sigma_d^2 (\sigma_d^2 + \sigma_s^2 \| \boldsymbol{x}\|^2 )}} \,. 
\end{eqnarray}

In addition, through this theory, we propose a \emc{new descriptor}. This descriptor is \emc{exact} in terms of what this theory suggests. In addition, this descriptor is derived from an \emc{affine matching} formulation, hence may better tolerate affine transforms than SIFT and DSP-SIFT\footnote{SIFT descriptor gains robustness against displacement by pooling across it. DSP-SIFT gains further robustness against scaling by scale pooling. However, none is robust to affine transform, which would require pooling across more parameters.}. Interestingly, despite handling a broader transformation space, it has fewer parameters than DSP-SIFT. Finally, it is \emc{analytical} and does not need any sampling.

\begin{eqnarray}
h_{heat} (\beta , \boldsymbol{x} ) &\triangleq& \int_{\mathcal{X}} \frac{e^{-\frac{(\boldsymbol{y}^T \tilde{\nabla} f(\boldsymbol{y}))^2}{2 \sigma_d^2 }}  w\big(- \frac{1}{2t} \, \tilde{\nabla}^T f(\boldsymbol{y})  \,(\, \sigma_d^{-2} \, \boldsymbol{y} \boldsymbol{x}^T \,+\, \sigma_a^{-2} \, \boldsymbol{I} \,\big)\, \tilde{\boldsymbol{v}}(\beta, \boldsymbol{y}) \big)}{\| \nabla f(\boldsymbol{y} ) \|^2 \, t^3} \nonumber \\
& & \quad \times k_{\sqrt{\sigma_d^2 + \sigma_a^2 \, \| \boldsymbol{x} \|^2}}(\frac{(\nabla f (\boldsymbol{y}))^T (\boldsymbol{x}- \boldsymbol{y})^\perp}{ \| \nabla f (\boldsymbol{y})\| })  \, d \boldsymbol{y} \,,
\end{eqnarray}

where $\tilde{\nabla} f (\boldsymbol{y}) \triangleq \frac{\nabla f(\boldsymbol{y})}{\| \nabla f(\boldsymbol{y})\|}$, $\tilde{\boldsymbol{v}}(\beta, \boldsymbol{y})\triangleq \frac{(\cos(\beta), \sin(\beta))}{\| \nabla f(\boldsymbol{y})\|}$, $t \triangleq \sqrt{\frac{(\boldsymbol{x}^T \tilde{\boldsymbol{v}}(\beta, \boldsymbol{y}))^2}{2 \sigma_d^2}+\frac{1}{2 \sigma_a^2 \, \| \nabla f(\boldsymbol{y}) \|^2}}$, $w(x) \triangleq \sqrt{\pi} e^{x^2} (1+2 x^2 ) \erfc(x) -2 x $, and $(a,b)^\perp \triangleq (b,-a)$.

Note that compared to DSP-SIFT, this descriptor \emc{reduces number of parameters} from two ($\sigma_r$ and $\sigma_s$) to one ($\sigma_a)$.

An illustration of how $h_{heat}$ differs from $h_{SIFT}$ is as follows. Consider a pair of images, namely image 1 and image 2, each consisting of two patches returned by some key point detector. The goal is to establish correspondence between patches using the $\ell_2$ distance between normalized descriptors,

\begin{equation}
d ( h_1 \,,\, h_2 ) \triangleq \int_0^{2 \pi} \int_{\mathcal{X}} \Big( \frac{h_1(\beta , \boldsymbol{x})}{\big(\int_0^{2 \pi} \int_{\mathcal{X}} h_1^2(\beta^\dag , \boldsymbol{x}^\dag) \, d \boldsymbol{x}^\dag \, d \beta^\dag\big)^{\frac{1}{2}}} - \frac{h_2(\beta , \boldsymbol{x})}{\big( \int_0^{2 \pi} \int_{\mathcal{X}} h_2^2(\beta^\dag , \boldsymbol{x}^\dag) \, d \boldsymbol{x}^\dag \, d \beta^\dag\big)^{\frac{1}{2}}} \Big)^2 \, d \boldsymbol{x} \, d \beta \,,
\end{equation}

where $\mathcal{X} \triangleq \mathcal{X}_1 \cap \mathcal{X}_2$. There are two possible matches: $P_A^1 \leftrightarrow P_A^2 \wedge P_B^1 \leftrightarrow P_B^2$ or $P_A^1 \leftrightarrow P_B^2 \wedge P_B^1 \leftrightarrow P_A^2$; obviously only the former is correct. Distance of matches using SIFT are listed in table \ref{tab:dist_SIFT_vs_Heat}. Note that SIFT descriptor attains lower distance for the wrong match and thus fails, while the heat descriptor finds the correct match. A visualization of SIFT descriptor and heat descriptor are presented in Figures \ref{fig:filter_response_sift} and \ref{fig:filter_response_heat} respectively.

\begin{table}
\begin{center}
\begin{tabular}{|c|c|c|}
\hline
& {\bf Correct}: $P_A^1 \leftrightarrow P_A^2 \wedge P_B^1 \leftrightarrow P_B^2$ & {\bf Wrong}: $P_A^1 \leftrightarrow P_B^2 \wedge P_B^1 \leftrightarrow P_A^2$\\
\hline\hline
SIFT & 0.20 & \bf 0.11 \\
\hline\hline
Heat & \bf 1.02 & 1.19 \\
\hline
\end{tabular}
\label{tab:dist_SIFT_vs_Heat}
\caption{Table shows total distance between wrongly matched patches and correctly matched patches. Correctly matched patches need to attain lower distance. SIFT fails to do that in this example, while the new descriptor succeeds.} 
\end{center}
\end{table}

\begin{figure}
\centering
\begin{tabular}{c c c c c }
\includegraphics[width=1in,height=1in]{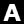} & \includegraphics[width=1in,height=1in]{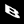} &
\hspace{0.2in} &
\includegraphics[width=1in,height=1in]{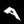} & \includegraphics[width=1in,height=1in]{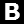}
\end{tabular}
\caption{The two patches on the left are considered to belong to image 1, and on the right to image 2.}
\label{fig:letters}
\end{figure}

\begin{figure}
\centering
\begin{tabular}{c c c c c c c c c}
\includegraphics[width=.4in,height=.4in]{img1_A.png} & \includegraphics[width=.4in,height=.4in]{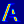} & \includegraphics[width=.4in,height=.4in]{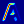} & \includegraphics[width=.4in,height=.4in]{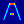} & \includegraphics[width=.4in,height=.4in]{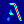} & \includegraphics[width=.4in,height=.4in]{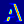} & \includegraphics[width=.4in,height=.4in]{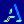} & \includegraphics[width=.4in,height=.4in]{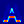} &
\includegraphics[width=.4in,height=.4in]{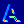} \\
\includegraphics[width=.4in,height=.4in]{img2_A.png} & \includegraphics[width=.4in,height=.4in]{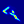} & \includegraphics[width=.4in,height=.4in]{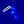} & \includegraphics[width=.4in,height=.4in]{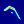} & \includegraphics[width=.4in,height=.4in]{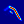} & \includegraphics[width=.4in,height=.4in]{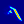} & \includegraphics[width=.4in,height=.4in]{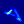} & \includegraphics[width=.4in,height=.4in]{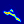} &
\includegraphics[width=.4in,height=.4in]{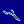} \\
\includegraphics[width=.4in,height=.4in]{img2_B.png} & \includegraphics[width=.4in,height=.4in]{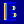} & \includegraphics[width=.4in,height=.4in]{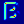} & \includegraphics[width=.4in,height=.4in]{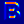} & \includegraphics[width=.4in,height=.4in]{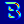} & \includegraphics[width=.4in,height=.4in]{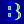} & \includegraphics[width=.4in,height=.4in]{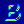} & \includegraphics[width=.4in,height=.4in]{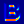} &
\includegraphics[width=.4in,height=.4in]{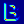} \\
\includegraphics[width=.4in,height=.4in]{img1_B.png} & \includegraphics[width=.4in,height=.4in]{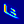} & \includegraphics[width=.4in,height=.4in]{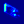} & \includegraphics[width=.4in,height=.4in]{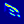} & \includegraphics[width=.4in,height=.4in]{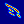} & \includegraphics[width=.4in,height=.4in]{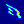} & \includegraphics[width=.4in,height=.4in]{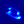} & \includegraphics[width=.4in,height=.4in]{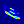} &
\includegraphics[width=.4in,height=.4in]{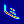} \\
Image & $\beta=0$ & $\beta=\frac{\pi}{8}$ & $\beta=\frac{2\pi}{8}$ & $\beta=\frac{3\pi}{8}$ & $\beta=\frac{4\pi}{8}$ & $\beta=\frac{5\pi}{8}$ & $\beta=\frac{6\pi}{8}$ & $\beta=\frac{7\pi}{8}$
\end{tabular}
\caption{Response map $h_{SIFT}(\beta,\,.\,)$ for different choices of $\beta$.}
\label{fig:filter_response_sift}
\end{figure}

\begin{figure}
\centering
\begin{tabular}{c c c c c c c c c}
\includegraphics[width=.4in,height=.4in]{img1_A.png} & \includegraphics[width=.4in,height=.4in]{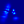} & \includegraphics[width=.4in,height=.4in]{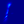} & \includegraphics[width=.4in,height=.4in]{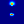} & \includegraphics[width=.4in,height=.4in]{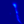} & \includegraphics[width=.4in,height=.4in]{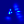} & \includegraphics[width=.4in,height=.4in]{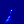} & \includegraphics[width=.4in,height=.4in]{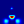} &
\includegraphics[width=.4in,height=.4in]{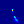} \\
\includegraphics[width=.4in,height=.4in]{img2_A.png} & \includegraphics[width=.4in,height=.4in]{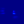} & \includegraphics[width=.4in,height=.4in]{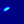} & \includegraphics[width=.4in,height=.4in]{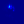} & \includegraphics[width=.4in,height=.4in]{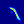} & \includegraphics[width=.4in,height=.4in]{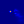} & \includegraphics[width=.4in,height=.4in]{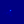} & \includegraphics[width=.4in,height=.4in]{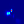} &
\includegraphics[width=.4in,height=.4in]{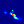} \\
\includegraphics[width=.4in,height=.4in]{img2_B.png} & \includegraphics[width=.4in,height=.4in]{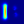} & \includegraphics[width=.4in,height=.4in]{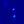} & \includegraphics[width=.4in,height=.4in]{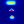} & \includegraphics[width=.4in,height=.4in]{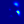} & \includegraphics[width=.4in,height=.4in]{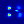} & \includegraphics[width=.4in,height=.4in]{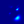} & \includegraphics[width=.4in,height=.4in]{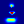} &
\includegraphics[width=.4in,height=.4in]{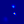} \\
\includegraphics[width=.4in,height=.4in]{img1_B.png} & \includegraphics[width=.4in,height=.4in]{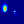} & \includegraphics[width=.4in,height=.4in]{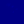} & \includegraphics[width=.4in,height=.4in]{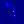} & \includegraphics[width=.4in,height=.4in]{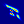} & \includegraphics[width=.4in,height=.4in]{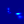} & \includegraphics[width=.4in,height=.4in]{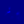} & \includegraphics[width=.4in,height=.4in]{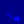} &
\includegraphics[width=.4in,height=.4in]{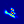} \\
Image & $\beta=0$ & $\beta=\frac{\pi}{8}$ & $\beta=\frac{2\pi}{8}$ & $\beta=\frac{3\pi}{8}$ & $\beta=\frac{4\pi}{8}$ & $\beta=\frac{5\pi}{8}$ & $\beta=\frac{6\pi}{8}$ & $\beta=\frac{7\pi}{8}$
\end{tabular}
\caption{Response map $h_{heat}(\beta,\,.\,)$ for different choices of $\beta$.}
\label{fig:filter_response_heat}
\end{figure}

Although this work focuses on SIFT, our diffusion theory can possible relate and extend other descriptors as well. For example, the recently developed distribution fields \cite{Erik13} are similar to (\ref{eq:sift}) and (\ref{eq:dsp_sift}), except that instead of histogram of gradient orientation, the histogram of intensity values are used,

\begin{equation}
h_{DF}(l , \boldsymbol{x} ) \triangleq \int_{\mathcal{X}} k_{\sigma_l} (l - f (\boldsymbol{y}) ) \, k_{\sigma_d} (\boldsymbol{y} - \boldsymbol{x}) \,  d \boldsymbol{y} \,,
\end{equation}

where $\sigma_l$ determines the smoothing strength of pixel intensity values. Similar to SIFT arguments, the convolution $k_{\sigma_d}$ may correspond to diffusion w.r.t. translation, and thus diffusion w.r.t. larger class of transformation, e.g., affine, may lead to geometrically more robust descriptors. Such extensions of distribution fields are not studied in the report, but are subject of future research.

\section{Matching as Energy Minimization}

For clarity of presentation, we focus on a restricted matching setup with simplifying assumptions. Nevertheless, this setup has enough complexity to make the point on nonconvexity and diffusion.

\subsection{Problem Setup}

\paragraph{\bf Notation:} An image is a map of form $f:\mathcal{X} \rightarrow [0,1]$, where $\mathcal{X} \subset \mathbb{R}^2$. Similarly, a patch is $p:\mathcal{P} \rightarrow [0,1]$, where $\mathcal{P} \subseteq \mathcal{X}$, i.e. the map is defined over a subset of the domain $\mathcal{X}$.

\paragraph{\bf Assumptions:} Given a set of patches $p_k: \mathcal{X}_k \rightarrow [0,1]$ for $k=1,\dots,n$. We assume that one of these patches, indexed by $k^*$, appears somewhere in $f$ up to a geometric transformation $\tau^*: \mathcal{X}_{k^*} \rightarrow \mathcal{X}$ and some reasonable intensity noise\footnote{In this setting each patch $p_k$ may be called a \emc{template}.},
\begin{equation}
\exists (k^*,\tau^*) \,\, \forall \boldsymbol{x} \in \mathcal{X}_{k^*} \,\,;\,\, f \big( \tau^*(\boldsymbol{x}) \big) \approx p_{k^*}(\boldsymbol{x}) \,\,.
\end{equation}

\paragraph{\bf Objective:} The goal is to estimate $(k^*,\tau^*)$. For tractability, the space of $\tau$ is \emc{parameterized} by a vector $\boldsymbol{\theta}$. For mathematical convenience, we assume the noise effect is best minimized via $\ell_2$ discrepancy,
\begin{equation}
\label{eq:matching_opt_integer}
(k^*,\boldsymbol{\theta}^*) \triangleq \argmin_{(k,\boldsymbol{\theta})} \int_{\mathcal{X}_k} \Big( f \big( \tau(\boldsymbol{x} \,;\, \boldsymbol{\theta}) \big) - p_k (\boldsymbol{x}) \Big)^2 \, d \boldsymbol{x} \,.
\end{equation}

The tools we later use apply to continuous variables, while (\ref{eq:matching_opt_integer}) involves the integer variable $k$. However, we can equivalently rewrite the problem in the following continuous form,

\begin{eqnarray}
& & (\boldsymbol{c}^*,\boldsymbol{\theta}^*) \triangleq \argmin_{(\boldsymbol{c},\boldsymbol{\theta})} \sum_k c^2_k \int_{\mathcal{X}_k} \Big( f \big( \tau(\boldsymbol{x} \,;\, \boldsymbol{\theta}) \big) - p_k (\boldsymbol{x}) \Big)^2 \, d \boldsymbol{x} \nonumber\\
\label{eq:matching_opt_cons}
\mbox{s.t} & & \sum_k c_k=1 \quad,\quad \forall k \,;\, c_k(1-c_k)=0\,.
\end{eqnarray}

\subsection{Intractability}

Despite simplicity of the setup, estimation of $(k^*,\boldsymbol{\theta}^*)$ is generally intractable because the optimization problem (\ref{eq:matching_opt_cons}) is \emc{nonconvex}. Hence, local optimization methods may converge to a \emc{local minimum}. In the following, we illustrate this by a toy example. The example involves a univariate signal $f(x)$, a pair of univariate templates $p_1(x)$ and $p_2(x)$ and a translation transform $\tau$ so that $f(\tau(x , \theta)) \triangleq f(x-\theta)$. Thus, (\ref{eq:matching_opt_cons}) can be expressed as below, after eliminating $c_2$ by the equality constraint $c_1+c_2=1$,

\begin{eqnarray}
(c_1^*,\theta^*) \triangleq \argmin_{(c_1,\theta)} & & c_1^2 \int_{\mathcal{X}_1} \big( f (x-\theta) - p_1 (x) \big)^2 \, d \boldsymbol{x} + (1-c_1)^2 \int_{\mathcal{X}_2} \big( f (x-\theta) - p_2 (x) \big)^2 \, d \boldsymbol{x}\nonumber\\
\label{eq:matching_opt_cons_toy}
\mbox{s.t} & & c_1(1-c_1)=0\,.
\end{eqnarray}

The solution $c_1^*$ determines to which template $f$ belongs to; $p_1$ if $c_1^*=1$ and $p_2$ if $c_1^*=0$.

We proceed by choosing $f$, $p_1$, and $p_2$ as the blue, green, and red curves in Figure \ref{fig:signals}-a. Here $\mathcal{X}=[-2,2]$ and $\mathcal{X}_1 = \mathcal{X}_2 = [-1.2,1.2]$. The goal is to slide the blue curve to the left or right, such that it coincides with either the green or red curve. Recall from (\ref{eq:matching_opt_cons}) that matching error is examined only over the support of the templates (gray shade). As shown in Figure \ref{fig:signals}-b, by sliding $f$ to the left by $\theta=0.25$ units, a perfect match with the green curve is achieved. However, there is no way to attain similar match with the red curve. Thus, by inspection we know that $\boldsymbol{c}^*=(1,0)$ and $\theta=0.25$.

\begin{figure}
\centering
\begin{tabular}{c c}
\includegraphics[width=2.3in,height=1in]{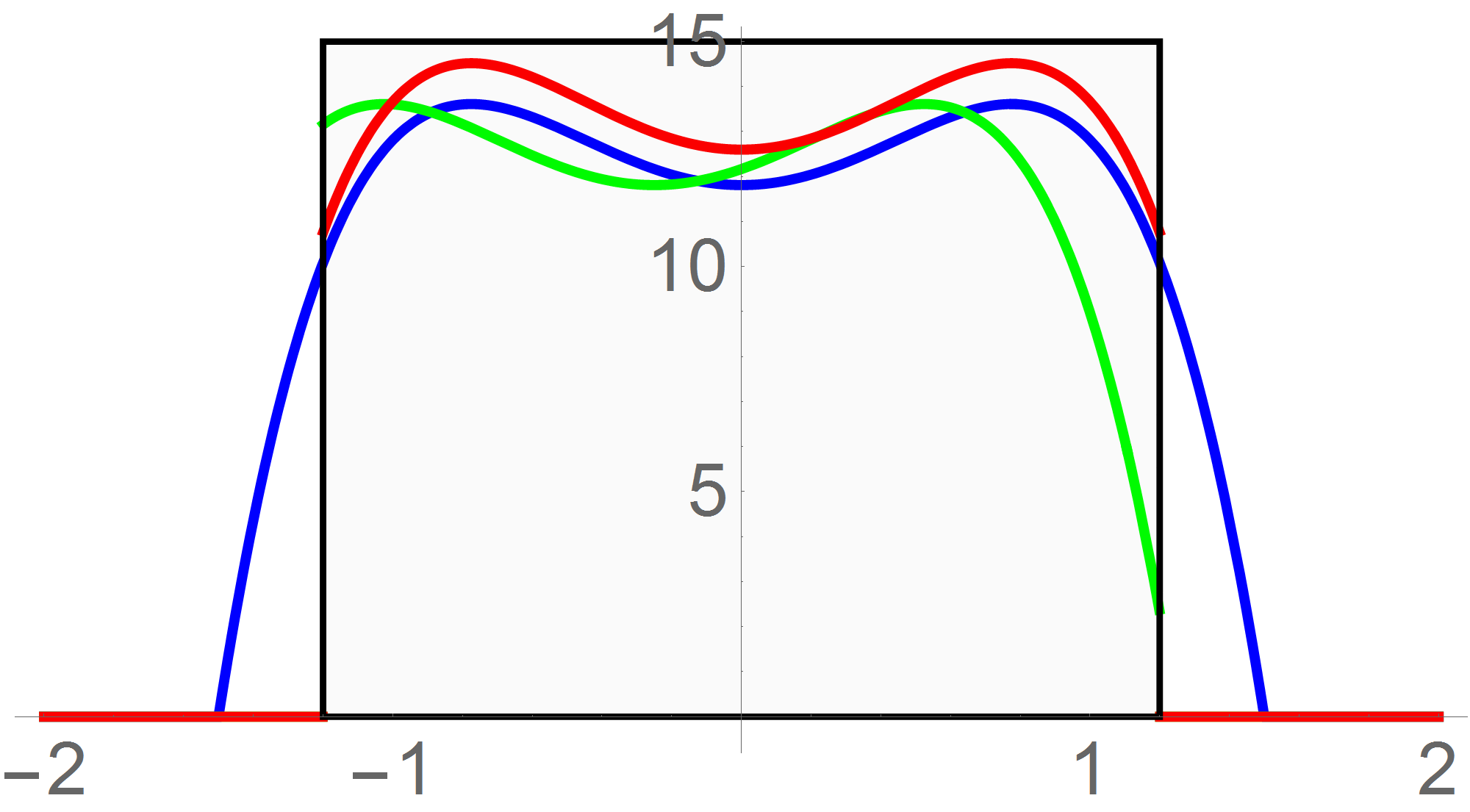} & \includegraphics[width=2.3in,height=1in]{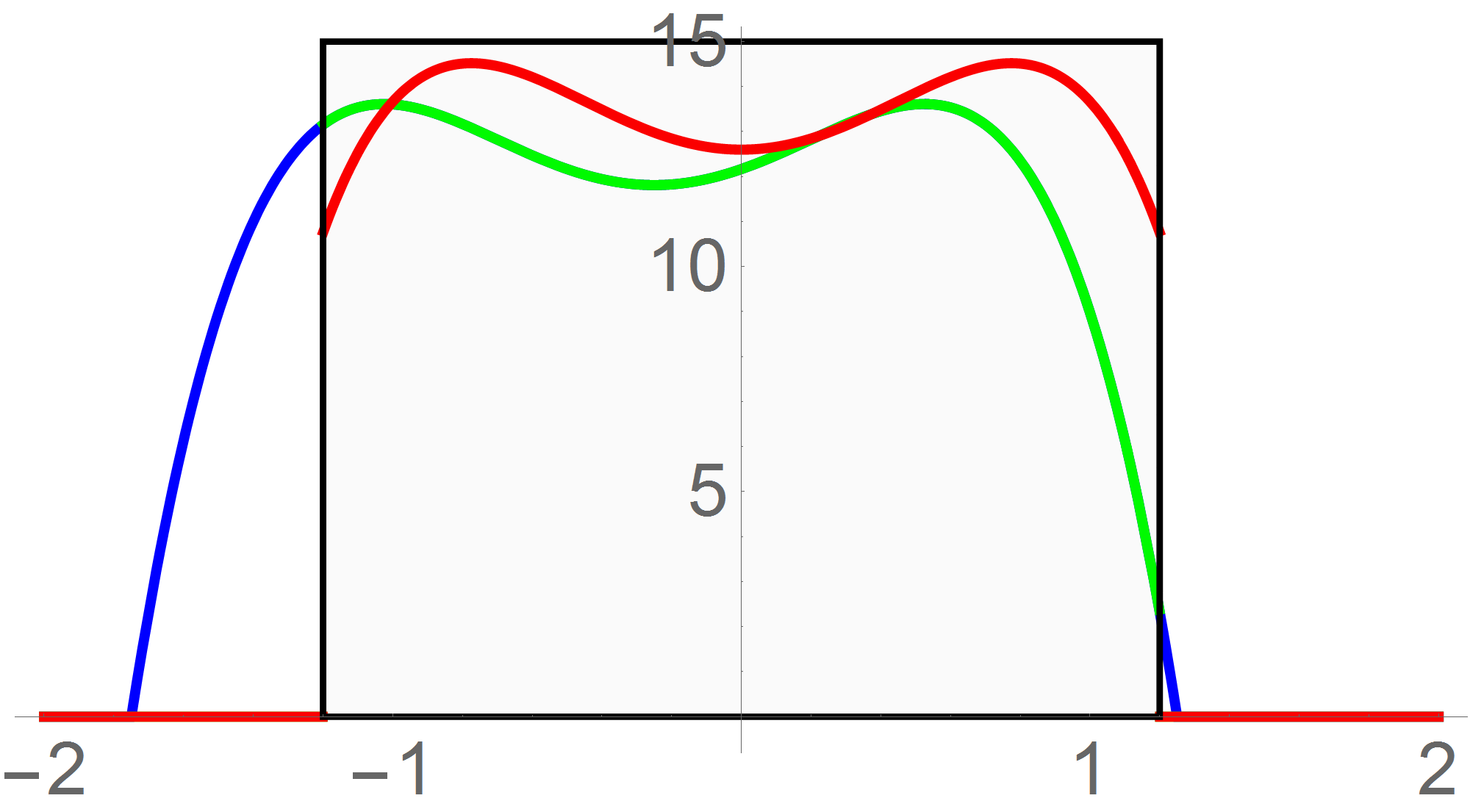}\\
(a) & (b)
\end{tabular}
\caption{Toy example of signal matching through shift.}
\label{fig:signals}
\end{figure}

For visualization purpose, we replace the equality in (\ref{eq:matching_opt_cons_toy}) by a quadratic penalty\footnote{Similar local minima could be obtained for the exact constrained optimization (\ref{eq:matching_opt_cons_toy}) using Lagrange multiplier technique.}. This encompasses both the objective and constraint into a single objective to be visualized. The resulted optimization landscape is shown in Figure \ref{fig:init_landscape}. A local minimum is apparent around $c_1=0, \theta=0$ while the global minimum is around $c_1=1,\theta=0.25$.

\begin{figure}
\centering
\begin{tabular}{c c}
\includegraphics[width=2.3in,height=1in]{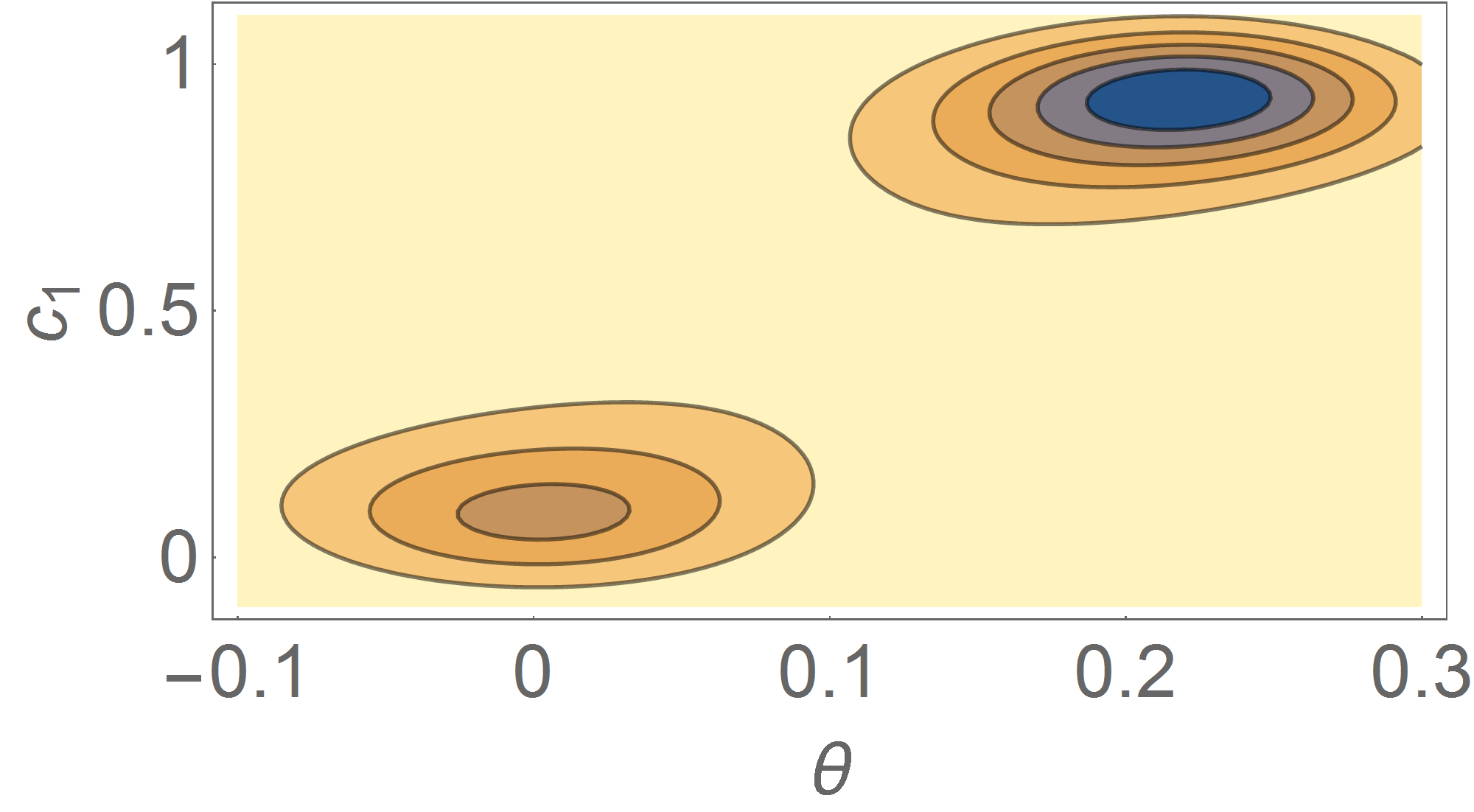} & \\
(a) & (b)
\end{tabular}
\caption{Objective landscape for the toy problem: signal matching through shift.}
\label{fig:init_landscape}
\end{figure}

\section{Diffusion}

One way to approximate the solution of a nonconvex optimization problem is by diffusion and the continuation method. The idea is to \emc{follow} the minimizer of the \emc{diffused} cost function while progressively transforming that function to the original nonconvex cost. It has recently been shown that, this procedure with the choice of the \emc{heat kernel} as the diffusion operator, provides the optimal transformation in a certain sense\footnote{It is shown that Gaussian convolution is resulted by the best affine approximation to a nonlinear PDE that generates the convex envelope. Note that computing the convex envelope of a function is generally intractable as well. Thus, it is not surprising that the associated nonlinear PDE lacks a closed form solution. However, by replacing the nonlinear PDE by its best affine approximation, we strike the optimal balance between tractability (closed form solution for the linear PDE) and accuracy of the approximation. The motivation for approximation the convex envelope is that the latter is an optimal object in several senses for the original nonconvex cost function. In particular, global minima of a nonconvex cost are contained in the global minima of its convex envelope.} \cite{MobahiEMMCVPR15}. In fact, some performance guarantees have been recently developed for this scheme \cite{Mobahi15AAAI}. The procedure is defined more formally below. Given an \emc{unconstrained} and \emc{nonconvex} cost function $h:\mathbb{R}^n \rightarrow \mathbb{R}$ to be minimized. Instead of applying a local optimization algorithm directly to $h$, we embed $h$ into a family of functions parameterized by $\sigma$,

\begin{equation}
\label{eq:gauss_smooth_cost}
g(\boldsymbol{x} \,;\, \sigma) \triangleq [h \star k_\sigma] (\boldsymbol{x})\,,
\end{equation}

where $\star$ is the convolution operator and $k_\sigma(\boldsymbol{x})$ is the Gaussian function with zero mean and covariance $\sigma^2 \boldsymbol{I}$. The Gaussian convolution appears here due to the known analytical solution form of the heat diffusion. Observe that $\lim_{\sigma \rightarrow \infty} g(\,.\, ;\sigma) = h(\boldsymbol{x})$. Thus by starting from a large $\sigma$ and shrinking it toward zero, a sequence of cost function converging to $h$ is obtained. The optimization process then follows the path of the minimizer of $g(\,.\, ; \sigma)$ through this sequence as listed in Algorithm \ref{alg:alg_goal}.

\begin{algorithm} [t]
\caption{Optimization by Diffusion and Continuation}
\label{alg:alg_goal}
\begin{algorithmic} [1]
\STATE Input: $f:\mathcal{X} \rightarrow \mathbb{R}$, Sequence $\sigma_0>\sigma_1>\dots>\sigma_n = 0$.
\STATE $\boldsymbol{x}_0=$ global minimizer of $g(\boldsymbol{x};\sigma_0)$.
\FOR {{$k=1$} \textbf{to} {$n$}}
\STATE $\boldsymbol{x}_{k}=$ Local minimizer of $g(\boldsymbol{x};\sigma_k)$, initialized at $\boldsymbol{x}_{k-1}$.
\ENDFOR
\STATE Output: $\boldsymbol{x}_n$
\end{algorithmic}
\end{algorithm}

Now let us revisit the problem (\ref{eq:matching_opt_cons_toy}). Like before, we use a quadratic penalty to obtain an unconstrained approximate to (\ref{eq:matching_opt_cons_toy}). A sequence of diffused landscapes of this problem is shown in Figure \ref{fig:annealed_landscaped}. Note that the problem becomes convex, with a unique strict minimizer at the large $\sigma$. The solution path originated from that point eventually lands at the global minimum in this example.

\begin{figure}
\centering
\begin{tabular}{c c c c}
\includegraphics[width=1.1in,height=1in]{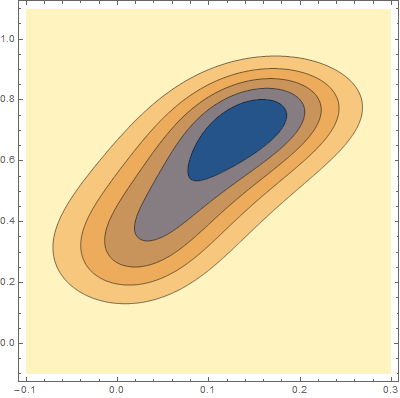} & 
\includegraphics[width=1.1in,height=1in]{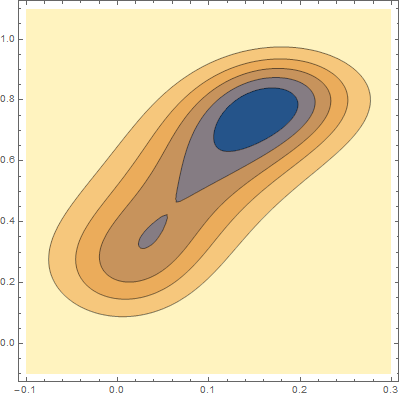} & \includegraphics[width=1.1in,height=1in]{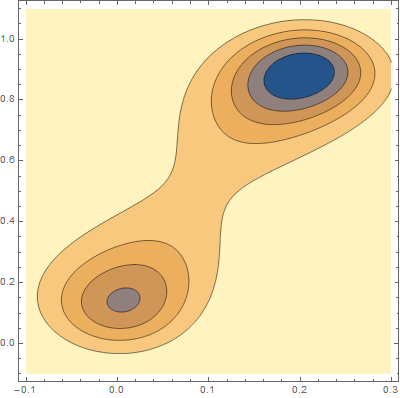} & 
\includegraphics[width=1.1in,height=1in]{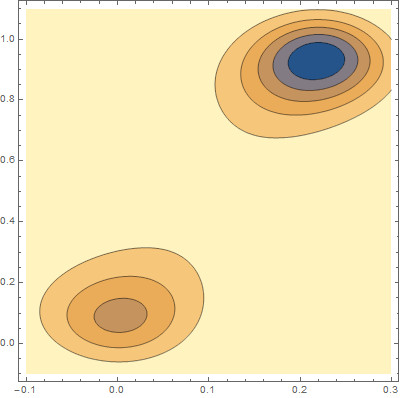}
\end{tabular}
\caption{From left to right: diffused cost functions from a large $\sigma$ toward $\sigma=0$ for the toy example.}
\label{fig:annealed_landscaped}
\end{figure}

\section{Deriving SIFT via the Diffusion Theory}
\label{sec:SIFT}

Instead of pixel intensity as (\ref{eq:matching_opt_cons}) to guide the matching, we switch to orientation of gradient. This change adds limited robustness to illumination changes \cite{Stefano15}. Nevertheless, the cost function remains nonconvex and difficult to minimize. Such nonconvex optimization may be treated via diffusion and continuation by the theory of \cite{MobahiEMMCVPR15}. We show that \emc{SIFT descriptor} emerges as an approximation to this process when $\tau$ is a \emc{similarity transformation}, i.e. $\tau(\boldsymbol{x} \,;\, \boldsymbol{\theta}) \triangleq e^s \boldsymbol{R}_\alpha \boldsymbol{x} + \boldsymbol{b}$, where $\boldsymbol{\theta} \triangleq (\alpha, s, \boldsymbol{b})$.

The approximation comes from two sources. First, the theory of \cite{MobahiEMMCVPR15} suggests a continuation method by gradually reducing $\sigma$ while following the path of the minimizer. SIFT provides an approximation to this process by solving the optimization at only one value of $\sigma$, i.e. it terminates after the first iteration of the algorithm suggested by \cite{MobahiEMMCVPR15}. Second, the cost function is diffused only w.r.t. a \emc{subset} of optimization variables ($\alpha$ and $\boldsymbol{b}$). This deviates from the theory in \cite{MobahiEMMCVPR15} that requires diffusion of the cost function in \emc{all variables}, i.e. to use $[cost \star k_\sigma](\boldsymbol{c}, \boldsymbol{\theta})$.

\subsection{Energy Function}

Define the density of gradient orientations of image $f$ as below,

\begin{equation}
\label{eq:def_gradient_orientation}
h (\beta , \boldsymbol{x} \,;\, f) \triangleq \Sh( \angle \nabla f(\boldsymbol{x}) - \beta ) \, \| \nabla f(\boldsymbol{x})\| \,,
\end{equation}

where $\Sh$ denotes the \emc{Dirac comb} of period $2\pi$, i.e. $\Sh(x) \triangleq \sum_{n=-\infty}^\infty \delta(x + 2 \pi n)$. The Dirac comb accounts for the periodicity of the angle (gradient orientation). Let the dissimilarity between a pair of density functions over a region $\mathcal{X}$ be expressed as the negated dot product,

\begin{eqnarray}
\label{eq:distance_functional}
d(f_1,f_2,\mathcal{X}) &\triangleq& -\int_0^{2\pi} \int_\mathcal{X} h (\beta , \boldsymbol{x} \,;\, f_1) \, \, h (\beta , \boldsymbol{x} \,;\, f_2) \,\, d \beta \, d \boldsymbol{x} \,.
\end{eqnarray}

In this setting, the problem of template matching is to find $\boldsymbol{\theta}$ such that the gradient orientations of $f ( \tau(\boldsymbol{x}\,;\, \boldsymbol{\theta}) )$ match that of a template,

\begin{eqnarray}
& & (\boldsymbol{c}^*,\boldsymbol{\theta}^*) \triangleq \argmin_{(\boldsymbol{c},\boldsymbol{\theta})} \sum_k c_k \, d(f \circ \tau_{\boldsymbol{\theta}} ,p_k, \mathcal{X}_k) \nonumber \\
\mbox{s.t} & & \sum_k c_k=1 \quad,\quad \forall k \,;\, c_k(1-c_k)=0 \,.
\end{eqnarray}

Replacing the equality constraint by some penalty function $q$ leads to the following \emc{unconstrained} optimization,

\begin{equation}
\label{eq:sift_uncons}
cost(\boldsymbol{c},\boldsymbol{\theta}) \triangleq q(\boldsymbol{c}) \,+\, \sum_k c_k \, d(f \circ \tau_{\boldsymbol{\theta}}, p_k,\mathcal{X}_k) \,.
\end{equation}

\subsection{Solution}
\begin{table}
\label{tab:side_by_side}

\begin{multicols}{3}
\tiny

\begin{algorithmic}[1]
\STATE Input: $\sigma_0>\sigma_1>\dots>\sigma_n = 0$.
\STATE $(\boldsymbol{\theta}_0,\boldsymbol{c}_0)= \argmin_{(\boldsymbol{\theta},\boldsymbol{c})}[cost \star k_{\sigma_0}](\boldsymbol{c},\boldsymbol{\theta})$.
\FOR {{$k=1$} \textbf{to} {$n$}}
\STATE $(\boldsymbol{\theta}_k,\boldsymbol{c}_k)=$ Local min of $[cost \star k_{\sigma_k}](\boldsymbol{c},\boldsymbol{\theta})$, initialized at $(\boldsymbol{\theta}_{k-1},\boldsymbol{c}_{k-1})$.
\ENDFOR
\STATE Output: $(\boldsymbol{\theta}_n,\boldsymbol{c}_n)$
\end{algorithmic}

\vfill
\columnbreak

\begin{algorithmic}[1]
\STATE Input: $\sigma_0$, $\Theta \triangleq \cup_j \{(\alpha_j, s_j , \boldsymbol{b}_j)\}$.
\STATE $(\boldsymbol{\theta}_0,\boldsymbol{c}_0)=\argmin_{(\boldsymbol{c},\alpha, s , \boldsymbol{b})}[cost(\boldsymbol{c},\,.\, s \,.\,) \star k_{\sigma_0}](\alpha, \boldsymbol{b})$ s.t. $(\alpha, s , \boldsymbol{b}) \in \Theta$.
\STATE Output: $(\boldsymbol{\theta}_0,\boldsymbol{c}_0)$
\end{algorithmic}

\vfill
\columnbreak

\begin{algorithmic}[1]
\STATE Input: $\sigma_0$, $\Theta \triangleq \cup_j \{(\alpha_j, s_j , \boldsymbol{b}_j)\}$.
\STATE $(j^*,k^*)=\argmax_{j,k} \int_0^{2\pi}  \int_{\mathcal{X}_k^\dag} h (\beta , \boldsymbol{x} \,;\, p_k) h_{SIFT}(\beta, \boldsymbol{x}\,,;\, f_j) \, d \boldsymbol{x} \, d \beta$ s.t. $(\alpha, s , \boldsymbol{b}) \in \Theta$.
\STATE Output: $(j^*,k^*)$
\end{algorithmic}

\end{multicols}
\caption{{\bf Left}: Ideal Minimization Strategy based on \cite{MobahiEMMCVPR15}. {\bf Middle}: Approximation due to a fixed $\sigma$ and partial diffusion. {\bf Right}: Equivalence with SIFT up to the approximation (\ref{eq:smoothed_sift}).}
\end{table}

The goal is to tackle the nonconvex problem (\ref{eq:sift_uncons}) using the diffusion and continuation theory of \cite{MobahiEMMCVPR15}. This would give the algorithm listed in Table \ref{tab:side_by_side}-left.

SIFT based matching can be derived by \emc{simplifying} this algorithm as described below,

\begin{itemize}
\item {\bf Partial Diffusion}: Instead of diffusion w.r.t. all variables $(\alpha,s,\boldsymbol{b},\boldsymbol{c})$, diffuse the energy function (\ref{eq:sift_uncons}) partially, i.e. only with respect to $(\alpha,\boldsymbol{b})$.

\item {\bf Fixed $\sigma$}: Instead of gradual refinement of the energy function by shrinking $\sigma$ toward zero, stick to a single choice $\sigma=\sigma_0$.

\item {\bf Limited Optimization}: Rather than searching the entire parameter space for $(\alpha, s , \boldsymbol{b})$ for the optimal solution, restrict to a small candidate set $\Theta \triangleq \cup_{j=1}^J \{ (\alpha_j, s_j , \boldsymbol{b}_j)\}$. This set is generated \emc{outside of the optimization} loop by a keypoint detector\footnote{The location and scale of candidate sets are determined by an \emc{interest point detector}, and the orientation angle is set to the \emc{dominant gradient direction}.}. Consequently, $\Theta$ does not necessarily contain the optimal parameter as keypoint estimation is done for each image in isolation and thus separately from the full matching problem. 

\end{itemize}

Applying these simplifications yields the algorithm in Table \ref{tab:side_by_side}-middle. The central optimization in this algorithm is the following,

\begin{equation}
\label{eq:sift_nested}
\min_{(\alpha, s , \boldsymbol{b}) \in \Theta} \min_{\boldsymbol{c}} \big[\, [cost \big(\boldsymbol{c},(\,.\,,s,\,.\,) \big) \star k_{\sigma_d}](\boldsymbol{b})\, \star \tilde{k}_{\sigma_r}\, \big] (\alpha) \,,
\end{equation}

where we have replaced the joint optimization $\min_{\boldsymbol{c},(\alpha, s , \boldsymbol{b}) \in \Theta}$ by the equivalent nested form $\min_{(\alpha, s , \boldsymbol{b}) \in \Theta} \min_{\boldsymbol{c}}$. The \emc{outer} minimization is trivial; it just loops over the candidates and evaluates the resulted cost to pick the best one. Below we only focus on the \emc{inner} optimization.

Assuming the penalty function $q(\boldsymbol{c})$ accurately enforces the constraint $c_k \in \{0,1\}$, the inner optimization becomes a \emc{winner take all} problem; the winning patch $p_k$ to match $f$ is the one which minimizes the following cost,

\begin{eqnarray}
\label{eq:sift_uncons_winner}
k^* &=& \argmin_k \big[\, [d(f \circ \tau_{(\,.\, , s , \,.\,)} ,p_k, \mathcal{X}_k) \star k_{\sigma_d}](\boldsymbol{b})\, \star \tilde{k}_{\sigma_r}\, \big] (\alpha) \\
&=& \argmax_k \int_0^{2\pi} \int_{\mathcal{X}_k} \Big( \big[\, [  h (\beta , \boldsymbol{x} \,;\, f \circ \tau_{(\,.\, , s , \,.\,)}) \star k_{\sigma_d}](\boldsymbol{b})\, \star \tilde{k}_{\sigma_r}\, \big] (\alpha) \Big) \, \, h (\beta , \boldsymbol{x} \,;\, p_k) \,\, d \beta \, d \boldsymbol{x} \,. \nonumber
\end{eqnarray}

We doubt that the convolutions in (\ref{eq:sift_uncons_winner}) are computationally tractable\footnote{We will later show in Section \ref{sec:affine} that by a different \emc{parameterization} of the geometric transform, we can handle a \emc{larger} class, namely the \emc{affine} transform, and yet are able to derive a \emc{closed form} expression for the convolution integrals.}. Thus, in order to derive a computationally tractable algorithm, we resort to a closed form approximation to the above convolutions. The approximation is stated in the following lemma.

\begin{lemma}

The following approximation holds,

\begin{eqnarray}
\label{eq:smoothed_sift}
& & \big[\, [  h (\beta , \boldsymbol{x} \,;\, f \circ \tau_{(\,.\, , s , \,.\,)}) \star k_{\sigma_d}](\boldsymbol{b})\, \star \tilde{k}_{\sigma_r}\, \big] (\alpha) \nonumber \\
&\approx&  \,-\, e^s \, \int_{\mathbb{R}^2} \tilde{k}_{\sigma_r} ( \angle \nabla f (\boldsymbol{y} ) - \alpha - \beta) \| \nabla f ( \boldsymbol{y} ) \| \, k_{\sigma_d} (\boldsymbol{y} - e^{s} \boldsymbol{R}_\alpha \boldsymbol{x} - \boldsymbol{b}) \, d \boldsymbol{y} \,. \nonumber
\end{eqnarray}
\end{lemma}

\begin{proof}
See Appendix \ref{sec:lemma_1} for the proof.
\end{proof}

Using this lemma, the computationally intractable optimization (\ref{eq:sift_nested}) is replaced by the following tractable approximation,

\begin{eqnarray}
\label{eq:sift_nested_approx}
& & \max_{(\alpha, s , \boldsymbol{b}) \in \Theta} \max_k \, e^s \, \int_0^{2\pi} \int_{\mathcal{X}_k} \int_{\mathbb{R}^2} \tilde{k}_{\sigma_r} ( \angle \nabla f (\boldsymbol{y} ) - \alpha - \beta) \| \nabla f ( \boldsymbol{y} ) \| \, k_{\sigma_d} (\boldsymbol{y} - e^{s} \boldsymbol{R}_\alpha \boldsymbol{x} - \boldsymbol{b}) \, d \boldsymbol{y} \nonumber \\
& & \hspace{1.8in} \times h (\beta, \boldsymbol{x} \,;\, p_k) \, d \boldsymbol{x} \, d \beta \,,
\end{eqnarray}

In the inner optimization, since $(\alpha, s , \boldsymbol{b})$ is fixed (to some $(\alpha_j, s_j , \boldsymbol{b}_j)$), the image $f$ can be \emc{warped} prior to optimization by $\tau(\,.\,;\alpha_j, s_j , \boldsymbol{b}_j)$. This allows optimization w.r.t. $k$ to be performed for $\tau$ being the \emc{identity transform} (because the effect of $(\alpha_j, s_j , \boldsymbol{b}_j)$ is already taken care of by the warp)\footnote{In this section we do not consider the full optimization loop (shrinking $\sigma$). However, if we wanted to do so, the idea of 1.Gradual reduction of the blur $\sigma$ and 2.Warping by the current estimate of the geometric transform in each iteration, would lead to a \emc{Lucas-Kanade}\cite{Lucas81} type algorithm. However, the resulted algorithm performs gradient density matching instead of Lucas-Kanade that relies on pixel intensity matching.}, i.e. $(\alpha=0, s=0 , \boldsymbol{b}=\boldsymbol{0})$. Denoting the warped $f$ due to $(\alpha_j, s_j , \boldsymbol{b}_j)$ by $f_j \triangleq f \circ \tau_{\boldsymbol{\theta}_j}$, the inner optimization simplifies,

\begin{eqnarray}
\label{eq:sift_inner_at_identity}
& & \max_k \, \int_0^{2\pi} \int_{\mathcal{X}_k} \underbrace{\int_{\mathbb{R}^2} \tilde{k}_{\sigma_r} ( \angle \nabla f_j (\boldsymbol{y} ) - \beta) \| \nabla f_j ( \boldsymbol{y} ) \| \, k_{\sigma_d} (\boldsymbol{y} - \boldsymbol{x}) \, d \boldsymbol{y}}_{h_{SIFT}(\beta, \boldsymbol{x}\,;\, f_j)} \, h (\beta, \boldsymbol{x} \,;\, p_k) \, d \boldsymbol{x} \, d \beta \,,
\end{eqnarray}

Part of the computation involving $f_j$ is \emc{independent of $p_k$} and can be precomputed. This precomputed result in fact provides a new \emc{representation} for $f_j$ that matches the definition of $h_{SIFT}$ in (\ref{eq:sift}). This gives the algorithm presented in Table \ref{tab:side_by_side}-Right.

\section{Deriving DSP-SIFT via the Diffusion Theory}

Here we show that the \emc{DSP-SIFT} descriptor also relates to partial diffusion of the cost function. Specifically, this descriptor can be derived by considering the diffusion w.r.t. the transformation parameters $(\alpha, s, \boldsymbol{b})$. Note that this involves more of the optimization variables in diffusion compared to SIFT (which diffuses w.r.t. $(\alpha, \boldsymbol{b})$), and thus provides a better approximation to the theory of \cite{MobahiEMMCVPR15} which suggests the diffusion must be applied to all optimization variables, i.e. to $(\boldsymbol{c},\boldsymbol{\theta})$. This improvement in approximation fidelity could be an explanation of why DSP-SIFT works better than SIFT in practice. However, it still misses diffusion of $\boldsymbol{c}$. 

The derivation is quite similar to that of SIFT in Section \ref{sec:SIFT}. The is to minimize the same energy function in (\ref{eq:sift_uncons}). However, on top of diffusion w.r.t. variables $(\alpha,\boldsymbol{b})$ we add a Gaussian convolution in $s$. By linearity of convolution, we can just take the diffused energy we obtained in (\ref{eq:sift_nested_approx}) and put it under the Gaussian convolution in $s$. This leads to the following expression,

\begin{eqnarray}
\label{eq:dsp_sift_by_smoothing}
& & \max_{(\alpha, s , \boldsymbol{b}) \in \Theta} \max_k \int_0^{2\pi} \int_{\mathcal{X}_k} \int_{\mathbb{R}^2} [\Big(\, e^{\,.\,} \, \tilde{k}_{\sigma_r} ( \angle \nabla f (\boldsymbol{y} ) - \alpha - \beta) \| \nabla f ( \boldsymbol{y} ) \| \, k_{\sigma_d} (\boldsymbol{y} - e^{\,.\,} \boldsymbol{R}_\alpha \boldsymbol{x} - \boldsymbol{b}) \nonumber \\
& & \hspace{1.8in} \Big) \star k_{\sigma_s}](s) \,\, \, d \boldsymbol{y} \, h (\beta, \boldsymbol{x} \,;\, p_k)\, d \boldsymbol{x} \, d \beta \,.
\end{eqnarray}

\section{Implication for Future Algorithms}

\subsection{Closed Form Approximations for Domain Size Pooling}

In DSP-SIFT, pooling over the scale is done numerically via sampling \cite{Stefano15}. Our theory suggests that the scale pooling should also be performed by Gaussian convolution, i.e. (\ref{eq:dsp_sift_by_smoothing}). Using this form, we present a closed-form approximation, which consequently does \emc{not require any sampling}. Whether or not this approximation provides a satisfactory fidelity must be investigated by experiments.

Recall energy minimization formulation of DSP-SIFT (\ref{eq:dsp_sift_by_smoothing}) has the following form,

\begin{eqnarray}
\label{eq:recall_dsp_sift_by_smoothing}
& & \max_{(\alpha, s , \boldsymbol{b}) \in \Theta} \max_k \int_0^{2\pi} \int_{\mathcal{X}_k} \int_{\mathbb{R}^2} \Big( [\big(\, e^{\,.\,} \,  \, k_{\sigma_d} (\boldsymbol{y} - e^{\,.\,} \boldsymbol{R}_\alpha \boldsymbol{x} - \boldsymbol{b}) \big) \star k_{\sigma_s}](s) \Big)\nonumber \\
& & \hspace{1.6in}  \, \tilde{k}_{\sigma_r} ( \angle \nabla f (\boldsymbol{y} ) - \alpha - \beta) \| \nabla f ( \boldsymbol{y} ) \| \,\, \, d \boldsymbol{y} \, h (\beta, \boldsymbol{x} \,;\, p_k)\, d \boldsymbol{x} \, d \beta \,.
\end{eqnarray}

This convolution does not have a closed form. However, we consider approximating $e^{s}$ by its linearized form around $s=0$ (identity scaling transform), i.e. $e^{s}\approx 1+s$. Then the convolution will have a closed form. Below we present two approximation based on this idea.

\paragraph{\bf Linearizing only the inner $e^s$}:

\begin{proposition}
\begin{eqnarray}
& & [\Big( e^{\,.\,} \, k_{\sigma_d} (\boldsymbol{y} - (1+{\,.\,}) \boldsymbol{R}_\alpha \boldsymbol{x} - \boldsymbol{b}) \Big) \star k_{\sigma_s}](s) \nonumber \\
&=& k_{\sigma_d}(\boldsymbol{y} - \boldsymbol{R}_\alpha \boldsymbol{x} - \boldsymbol{b}) \nonumber \\
& & \times k^{-1}_{\frac{\sigma_d}{\|\boldsymbol{x}\|}} (1+\frac{(\boldsymbol{R}_\alpha \boldsymbol{x})^T (\boldsymbol{b}- \boldsymbol{y}) - \sigma_d^2}{\| \boldsymbol{x}\|^2}) \times k_{\sqrt{\sigma_s^2+ \frac{\sigma_d^2}{ \|\boldsymbol{x}\|^2}}} (s+1 + \frac{(\boldsymbol{R}_\alpha \boldsymbol{x})^T (\boldsymbol{b} - \boldsymbol{y}) - \sigma_d^2}{\|\boldsymbol{x}\|^2}) \,. \nonumber
\end{eqnarray}

\end{proposition}

\begin{proof}

See Appendix \ref{sec:lemma_one_linearization} for the proof.

\end{proof}

In particular, when the region is already warped, we can set $(\alpha,s,\boldsymbol{b})=(0,0,\boldsymbol{0})$. This allows the template matching solution (\ref{eq:recall_dsp_sift_by_smoothing}) as below,

\mycolorbox{kerem}{
\begin{eqnarray}
& & \max_{j,k} \int_0^{2\pi} \int_{\mathcal{X}_k} \int_{\mathbb{R}^2} k_{\sigma_d}(\boldsymbol{y} - \boldsymbol{x}) \times k^{-1}_{\frac{\sigma_d}{\|\boldsymbol{x}\|}} (1-\frac{\boldsymbol{x}^T \boldsymbol{y} + \sigma_d^2}{\| \boldsymbol{x}\|^2}) \times k_{\sqrt{\sigma_s^2+ \frac{\sigma_d^2}{ \|\boldsymbol{x}\|^2}}} (1 - \frac{\boldsymbol{x}^T \boldsymbol{y} + \sigma_d^2}{\|\boldsymbol{x}\|^2}) \nonumber \\
& & \hspace{1in}  \times \, \tilde{k}_{\sigma_r} ( \angle \nabla f_j (\boldsymbol{y} ) - \beta) \| \nabla f_j ( \boldsymbol{y} ) \| \,\, \, d \boldsymbol{y} \, h (\beta, \boldsymbol{x} \,;\, p_k)\, d \boldsymbol{x} \, d \beta \,.
\end{eqnarray}
}

\paragraph{\bf Linearizing both the inner and outer $e^s$}:

We use the following identity,

\begin{eqnarray}
& & [ (1 + \,.\,) k_\sigma (\boldsymbol{y}+(1+ \,.\, )\boldsymbol{x}) \star k_{scale}](s)\\
&=& \frac{\sigma^2 (1 + s) - 
   \sigma_{scale} \boldsymbol{x}^T \boldsymbol{y}}{2\pi \sigma  (\sigma^2 + \sigma_{scale}^2 \| \boldsymbol{x}\|^2)^\frac{3}{2}} e^{-\frac{\sigma^2 \|(1 + s ) \boldsymbol{x} + \boldsymbol{y} \|^2 + \sigma_{scale}^2 (\boldsymbol{x}^T \boldsymbol{y}^\perp)^2}{
  2 \sigma^2 (\sigma^2 + \sigma_{scale}^2 \| \boldsymbol{x}\|^2 )}} 
\end{eqnarray}

Thus,

\begin{eqnarray}
& & [ (1 + \,.\,) k_{\sigma_d} (\boldsymbol{y} - \boldsymbol{b} -(1+ \,.\, ) \boldsymbol{R}_\alpha \, \boldsymbol{x}) \star k_{\sigma_s}](s)\\
&=& \frac{\sigma_d^2 (1 + s) + 
   \sigma_s \boldsymbol{x}^T \boldsymbol{R}^T_\alpha (\boldsymbol{y} - \boldsymbol{b})}{2\pi \sigma_d  (\sigma_d^2 + \sigma_s^2 \| \boldsymbol{x}\|^2)^\frac{3}{2}} e^{-\frac{\sigma_d^2 \|-(1 + s ) \boldsymbol{R}_\alpha \boldsymbol{x} + \boldsymbol{y} - \boldsymbol{b} \|^2 + \sigma_s^2 ((\boldsymbol{R}_\alpha \boldsymbol{x})^T (\boldsymbol{y}-\boldsymbol{b})^\perp )^2}{
  2 \sigma_d^2 (\sigma_d^2 + \sigma_s^2 \| \boldsymbol{x}\|^2 )}} \,.
\end{eqnarray}

In particular, when the region is already warped, we can set $(\alpha,s,\boldsymbol{b})=(0,0,\boldsymbol{0})$. This allows the template matching solution (\ref{eq:recall_dsp_sift_by_smoothing}) as below,

\mycolorbox{kerem}{
\begin{eqnarray}
& & \max_{j,k} \int_0^{2\pi} \int_{\mathcal{X}_k} \int_{\mathbb{R}^2} \frac{\sigma_d^2 + 
   \sigma_s \boldsymbol{x}^T \boldsymbol{y} }{(\sigma_d^2 + \sigma_s^2 \| \boldsymbol{x}\|^2)^\frac{3}{2}} e^{-\frac{\sigma_d^2 \|\boldsymbol{y} - \boldsymbol{x}\|^2 + \sigma_s^2 (\boldsymbol{x}^T \boldsymbol{y}^\perp )^2}{
  2 \sigma_d^2 (\sigma_d^2 + \sigma_s^2 \| \boldsymbol{x}\|^2 )}} \\
& & \hspace{1in}  \times \, \tilde{k}_{\sigma_r} ( \angle \nabla f_j (\boldsymbol{y} ) - \beta) \| \nabla f_j ( \boldsymbol{y} ) \| \,\, \, d \boldsymbol{y} \, h (\beta, \boldsymbol{x} \,;\, p_k)\, d \boldsymbol{x} \, d \beta \,.
\end{eqnarray}
}

\subsection{Exact Diffusion for Affine Transform}
\label{sec:affine}

Using the diffusion theory, by using a different \emc{parameterization} for the geometric transformation, we can potentially improved SIFT and DSP-SIFT in two ways.

\begin{enumerate}
\item We can extend the descriptor from handling \emc{similarity} transform to \emc{affine} transform.
\item Recall that the computation of the diffusion in SIFT and DSP-SIFT relies on some \emc{approximation}. In addition, regardless of the diffusion theory, DSP-SIFT involves sampling to approximate one of the required \emc{integrals}. The finite sampling process is inaccurate and expensive to compute. Here despite working with a larger transformation space, the new parameterization allows deriving \emc{exact} and \emc{closed form} expression for the diffusion in \emc{all} transformation parameters. 
\end{enumerate}

The formulation of the energy function is similar to that of SIFT and DSP-SIFT, except the parameterization. Instead of the similarity transform $\boldsymbol{\tau}$ as $\tau(\boldsymbol{x} \,;\, \alpha, s ,\boldsymbol{b}) \triangleq e^s \boldsymbol{R}_\alpha \boldsymbol{x} + \boldsymbol{b}$ we switch to the affine transform $\tau(\boldsymbol{x} \,;\, \boldsymbol{A} ,\boldsymbol{b}) \triangleq \boldsymbol{A} \boldsymbol{x} + \boldsymbol{b}$, as listed below,

\begin{equation}
\label{eq:affine_uncons}
cost(\boldsymbol{c},\boldsymbol{A},\boldsymbol{b}) \triangleq q(\boldsymbol{c}) \,+\, \sum_k c_k \, d(f \circ \tau_{\boldsymbol{A},\boldsymbol{b}}, p_k,\mathcal{X}_k) \,,
\end{equation}

where the dissimilarity functional $d$ is defined as earlier in (\ref{eq:distance_functional}). Recall that the goal is to tackle the nonconvex problem (\ref{eq:affine_uncons}) using the diffusion theory of \cite{MobahiEMMCVPR15} and similar simplifications as in Section \ref{sec:SIFT}, we obtain the following solution for the template matching problem.

\begin{eqnarray}
\label{eq:affine_uncons_winner}
k^* &=& \argmin_k \big[\, [d(f \circ \tau_{(\,.\, , \,.\,)} ,p_k, \mathcal{X}_k) \star k_{\sigma_b}](\boldsymbol{b})\, \star \tilde{k}_{\sigma_a}\, \big] (\boldsymbol{A}) \\
&=& \argmax_k \int_0^{2\pi} \int_{\mathcal{X}_k} \Big( \big[\, [  h (\beta , \boldsymbol{x} \,;\, f \circ \tau_{(\,.\, , \,.\,)}) \star k_{\sigma_b}](\boldsymbol{b})\, \star \tilde{k}_{\sigma_a}\, \big] (\boldsymbol{A}) \Big) \, \, h (\beta , \boldsymbol{x} \,;\, p_k) \,\, d \beta \, d \boldsymbol{x} \,. \nonumber
\end{eqnarray}

Interesting, we can replace the above convolutions by a \emc{closed form} and \emc{exact} expression. This is stated in the following lemma.

\begin{lemma}

\begin{eqnarray}
& & \big[ \big( [ h (\beta , \boldsymbol{x} \,;\, f \circ \tau_{(\,.\,,\,.\,)})  \star k_{\sigma_b}](\boldsymbol{b}) \big) \star k_{\sigma_a} \big](\boldsymbol{A}) \nonumber\\
&=& \frac{e^{-\frac{((\boldsymbol{b}-\boldsymbol{y})^T \tilde{\nabla} f(\boldsymbol{y}))^2}{2 \sigma_b^2 }-\frac{\|\boldsymbol{A}^T \tilde{\nabla} f (\boldsymbol{y})\|^2}{2 {\sigma_a}^2}}  w(- \frac{ \sigma_b^{-2} \tilde{\nabla}^T f(\boldsymbol{y}) (\boldsymbol{y}-\boldsymbol{b}) \boldsymbol{x}^T \tilde{\boldsymbol{v}} (\beta, \boldsymbol{y})+{\sigma_a}^{-2} \tilde{\nabla}^T f (\boldsymbol{y}) \boldsymbol{A} \tilde{\boldsymbol{v}}(\beta, \boldsymbol{y})  }{2 t})}{8 \sqrt{2} \pi^{\frac{3}{2}} \sigma_b {\sigma_a}^2 \, \| \nabla f(\boldsymbol{y}) \|^2 \, t^3} \nonumber \,, 
\end{eqnarray}

where $\tilde{\nabla} f (\boldsymbol{y}) \triangleq \frac{\nabla f(\boldsymbol{y})}{\| \nabla f(\boldsymbol{y})\|}$, $\tilde{\boldsymbol{v}}(\beta, \boldsymbol{y})\triangleq \frac{\boldsymbol{v}(\beta)}{\| \nabla f(\boldsymbol{y})\|}$, and $t \triangleq \sqrt{\frac{(\boldsymbol{x}^T \tilde{\boldsymbol{v}}(\beta, \boldsymbol{y}))^2}{2 \sigma_b^2}+\frac{1}{2 {\sigma_a}^2 \, \| \nabla f(\boldsymbol{y}) \|^2}}$ and $w(x) \triangleq \sqrt{\pi} e^{x^2} (1+2 x^2 ) \erfc(x) -2 x $.
\end{lemma}

\begin{proof}
See Appendix \ref{sec:lemma_2} for the proof.
\end{proof}

Similar to the arguments about SIFT solution in Section \ref{sec:SIFT}, the inner optimization in (\ref{eq:affine_uncons_winner}) can work with the warped $f$ so that the transformation $\tau(\boldsymbol{x} \,;\, \boldsymbol{A}, \boldsymbol{b})$ simplifies to the \emc{identity transform} $(\boldsymbol{A}=\boldsymbol{I}, \boldsymbol{b}=\boldsymbol{0})$. Letting the warped $f$ be $f_j \triangleq f \circ \tau_{\boldsymbol{A}_j, \boldsymbol{b}_j}$, the inner optimization simplifies,

where $h_{heat}$ is defined as the result in lemma 2 (diffused $h$) with $(\boldsymbol{A}=\boldsymbol{I}, \boldsymbol{b}=\boldsymbol{0})$, $\sigma_a$ and $\sigma_b$ fixed, and all constants dropped,

\mycolorbox{kerem}{
\begin{eqnarray}
(j^*,k^*)&\triangleq& \argmax_{j,k} \, \int_0^{2\pi} \int_{\mathcal{X}_k} h_{heat}(\beta, \boldsymbol{x}\,;\, f_j) \, h (\beta, \boldsymbol{x} \,;\, p_k) \, d \boldsymbol{x} \, d \beta \nonumber \\
h_{heat}(\beta, \boldsymbol{x}\,;\, f) &\triangleq& \frac{e^{-\frac{(\boldsymbol{y}^T \tilde{\nabla} f(\boldsymbol{y}))^2}{2 \sigma_b^2 }}  w\big(- \frac{1}{2t} \, \tilde{\nabla}^T f(\boldsymbol{y})  \,(\, \sigma_b^{-2} \, \boldsymbol{y} \boldsymbol{x}^T \,+\, {\sigma_a}^{-2} \, \boldsymbol{I} \,\big)\, \tilde{\boldsymbol{v}}(\beta, \boldsymbol{y}) \big)}{\| \nabla f(\boldsymbol{y} ) \|^2 \, t^3} \nonumber \\
\tilde{\nabla} f (\boldsymbol{y}) &\triangleq& \frac{\nabla f(\boldsymbol{y})}{\| \nabla f(\boldsymbol{y})\|} \nonumber \\
\tilde{\boldsymbol{v}}(\beta, \boldsymbol{y}) &\triangleq& \frac{\boldsymbol{v}(\beta)}{\| \nabla f(\boldsymbol{y})\|} \nonumber \\
t &\triangleq& \sqrt{\frac{(\boldsymbol{x}^T \tilde{\boldsymbol{v}}(\beta, \boldsymbol{y}))^2}{2 \sigma_b^2}+\frac{1}{2 {\sigma_a}^2 \, \| \nabla f(\boldsymbol{y}) \|^2}} \nonumber \\
w(x) &\triangleq& \sqrt{\pi} e^{x^2} (1+2 x^2 ) \erfc(x) -2 x \nonumber \,.
\end{eqnarray}
}

\bibliographystyle{apalike}
\bibliography{sift_theory}

\pagebreak
\section*{Appendix}
\appendix

\section{Proof of Proposition 1}
\label{sec:lemma_one_linearization}

We proceed with the following identity\footnote{
We essentially have $e^s \times k_\sigma( \boldsymbol{y}+(1+s) \boldsymbol{x})$ which by completing the square of the exponent w.r.t. $s$ can be expressed as below,

\begin{eqnarray}
& & e^s \times k_\sigma( \boldsymbol{y}+(1+s) \boldsymbol{x}) \\
&=& k_\sigma(\boldsymbol{x} + \boldsymbol{y}) \times k^{-1}_{\frac{\sigma }{\| \boldsymbol{x}\|}} (1+\frac{\boldsymbol{x}^T \boldsymbol{y} - \sigma^2}{\| \boldsymbol{x}\|^2}) \times k_{\frac{\sigma}{ \| \boldsymbol{x}\|}} (s+1 + \frac{\boldsymbol{x}^T \boldsymbol{y} - \sigma^2}{\| \boldsymbol{x}\|^2}) \,,
\end{eqnarray}

where the first $k$ is 2D, and the next two $k$'s are $1D$.
},

\begin{eqnarray}
& & e^s \times k_\sigma( \boldsymbol{y} -(1+s) \boldsymbol{R}_\alpha \boldsymbol{x} - \boldsymbol{b} ) \\
&=& k_\sigma(\boldsymbol{y} - \boldsymbol{R}_\alpha \boldsymbol{x} - \boldsymbol{b}) \times k^{-1}_{\frac{\sigma }{\| \boldsymbol{R}_\alpha  \boldsymbol{x}\|}} (1+\frac{(\boldsymbol{R}_\alpha \boldsymbol{x})^T (\boldsymbol{b}- \boldsymbol{y}) - \sigma^2}{\| \boldsymbol{R}_\alpha \boldsymbol{x}\|^2}) \times k_{\frac{\sigma}{ \|\boldsymbol{R}_\alpha  \boldsymbol{x}\|}} (s+1 + \frac{(\boldsymbol{R}_\alpha \boldsymbol{x})^T (\boldsymbol{b} - \boldsymbol{y}) - \sigma^2}{\| \boldsymbol{R}_\alpha \boldsymbol{x}\|^2}) \\
&=& k_\sigma(\boldsymbol{y} - \boldsymbol{R}_\alpha \boldsymbol{x} - \boldsymbol{b}) \times k^{-1}_{\frac{\sigma }{\| \boldsymbol{x}\|}} (1+\frac{(\boldsymbol{R}_\alpha \boldsymbol{x})^T (\boldsymbol{b}- \boldsymbol{y}) - \sigma^2}{\| \boldsymbol{x}\|^2}) \times k_{\frac{\sigma}{ \|\boldsymbol{x}\|}} (s+1 + \frac{(\boldsymbol{R}_\alpha \boldsymbol{x})^T (\boldsymbol{b} - \boldsymbol{y}) - \sigma^2}{\|\boldsymbol{x}\|^2}) \\
\end{eqnarray}

In this form, it is now very easy to compute convolution with $k_{\sigma_{{scale}}}(s)$,

\begin{eqnarray}
& & [e^{\,.\,} \times k_\sigma( \boldsymbol{y} -(1+\,.\,) \boldsymbol{R}_\alpha \boldsymbol{x} - \boldsymbol{b}) \star k_{\sigma_{{scale}}}](s) \\
&=& k_\sigma(\boldsymbol{y} - \boldsymbol{R}_\alpha \boldsymbol{x} - \boldsymbol{b}) \times k^{-1}_{\frac{\sigma }{\| \boldsymbol{x}\|}} (1+\frac{(\boldsymbol{R}_\alpha \boldsymbol{x})^T (\boldsymbol{b}- \boldsymbol{y}) - \sigma^2}{\| \boldsymbol{x}\|^2}) \times k_{\sqrt{\sigma^2_{{scale}}+ \frac{\sigma^2}{ \|\boldsymbol{x}\|^2}}} (s+1 + \frac{(\boldsymbol{R}_\alpha \boldsymbol{x})^T (\boldsymbol{b} - \boldsymbol{y}) - \sigma^2}{\|\boldsymbol{x}\|^2}) \\
\end{eqnarray}

Therefore,

\begin{eqnarray}
& & [\Big( e^{\,.\,} \int_{\mathbb{R}^2} \tilde{k}_{\tilde{\sigma}} ( \angle \nabla f (\boldsymbol{y} ) - \alpha - \beta) \| \nabla f ( \boldsymbol{y} ) \| \, k_\sigma (\boldsymbol{y} - (1+{\,.\,}) \boldsymbol{R}_\alpha \boldsymbol{x} - \boldsymbol{b}) \, d \boldsymbol{y} \Big) \star k_{\sigma_{{scale}}}](s) \\
&=& \int_{\mathbb{R}^2} \tilde{k}_{\tilde{\sigma}} ( \angle \nabla f (\boldsymbol{y} ) - \alpha - \beta) \| \nabla f ( \boldsymbol{y} ) \| \, k_\sigma(\boldsymbol{y} - \boldsymbol{R}_\alpha \boldsymbol{x} - \boldsymbol{b}) \\
& & \times k^{-1}_{\frac{\sigma }{\| \boldsymbol{x}\|}} (1+\frac{(\boldsymbol{R}_\alpha \boldsymbol{x})^T (\boldsymbol{b}- \boldsymbol{y}) - \sigma^2}{\| \boldsymbol{x}\|^2}) \times k_{\sqrt{\sigma^2_{{scale}}+ \frac{\sigma^2}{ \|\boldsymbol{x}\|^2}}} (s+1 + \frac{(\boldsymbol{R}_\alpha \boldsymbol{x})^T (\boldsymbol{b} - \boldsymbol{y}) - \sigma^2}{\|\boldsymbol{x}\|^2}) \, d \boldsymbol{y} \,. \nonumber
\end{eqnarray}

\pagebreak
\section{Proof of Lemma 1}
\label{sec:lemma_1}

\begin{eqnarray}
& & cost(\boldsymbol{c},\boldsymbol{\theta}) \\
&\triangleq& q(\boldsymbol{c}) \,+\, \sum_k c_k \, d(f \circ \tau_{(\alpha\,,s,\,\boldsymbol{b}\,)}, p_k, \mathcal{X}_k) \\
&=& q(\boldsymbol{c}) \,+\, \sum_k c_k \, \int_0^{2\pi} \Big( h (\beta\,;\, f \circ \tau_{(\alpha,s,\,\boldsymbol{b}\,)}, \mathcal{X}_k ) - h (\beta\,;\, p_k, \mathcal{X}_k ) \Big)^2 \, d \beta \,.
\end{eqnarray}

Note that,

\begin{eqnarray}
& & h (\beta\,;\, f \circ \tau_{(\alpha,s,\,\boldsymbol{b}\,)} , \mathcal{X}_k) \\
&=& \int_{\mathcal{X}_k} \Sh( \angle \nabla (f ( e^{s} \boldsymbol{R}_\alpha \boldsymbol{x} + \boldsymbol{b} )) - \beta) \| \nabla f ( e^{s} \boldsymbol{R}_\alpha \boldsymbol{x} + \boldsymbol{b} ) \| \, d \boldsymbol{x} \\
\label{eq:chain_rule}
&=& \int_{\mathcal{X}_k} \Sh( \angle e^{s} \boldsymbol{R}^T_\alpha [\nabla f] ( e^{s} \boldsymbol{R}_\alpha \boldsymbol{x} + \boldsymbol{b} ) - \beta)  \| e^{s} \boldsymbol{R}^T_\alpha \nabla f ( e^{s} \boldsymbol{R}_\alpha \boldsymbol{x} + \boldsymbol{b} ) \| \, d \boldsymbol{x} \\
\label{eq:identity}
&=& \int_{\mathcal{X}_k} \Sh( \angle \nabla f ( e^{s} \boldsymbol{R}_\alpha \boldsymbol{x} + \boldsymbol{b} ) - \alpha - \beta) e^{s} \| \nabla f ( e^{s} \boldsymbol{R}_\alpha \boldsymbol{x} + \boldsymbol{b} ) \| \, d \boldsymbol{x} \,,
\end{eqnarray}

where (\ref{eq:chain_rule}) uses the chain rule of derivative $
\nabla \big ( f( \boldsymbol{Ax}+\boldsymbol{b}) \big) = \boldsymbol{A}^T \big( [\nabla f] ( \boldsymbol{Ax}+\boldsymbol{b} ) \big)$. Also, for any $a>0$, (\ref{eq:identity}) uses the identities $\| a \boldsymbol{R} \boldsymbol{x}\| = a \| \boldsymbol{x}\|$ and $\angle a \boldsymbol{R}_\alpha \boldsymbol{x} = \alpha + \angle \boldsymbol{x}$. Thus, it follows that,

\begin{eqnarray}
& & cost(\boldsymbol{c},\boldsymbol{\theta}) \\
&=& q(\boldsymbol{c}) \,+\, \sum_k c_k \, \int_0^{2\pi} \Big( \int_{\mathcal{X}_k} \Sh( \angle \nabla f ( e^{s} \boldsymbol{R}_\alpha \boldsymbol{x} + \boldsymbol{b} ) - \alpha - \beta) e^{s} \|  \nabla f ( e^{s} \boldsymbol{R}_\alpha \boldsymbol{x} + \boldsymbol{b} ) \| \, d \boldsymbol{x} - h (\beta\,;\, p_k , \mathcal{X}_k ) \Big)^2 \, d \beta\,.
\end{eqnarray}

By the linearity of the convolution operator and the unity of Gaussian's total mass, smoothed $cost$ amounts only to replacing $\Big( \int_{\mathcal{X}_k} \Sh( \angle \nabla f ( e^{s} \boldsymbol{R}_\alpha \boldsymbol{x} + \boldsymbol{b} ) - \alpha - \beta) e^{s} \|  \nabla f ( e^{s} \boldsymbol{R}_\alpha \boldsymbol{x} + \boldsymbol{b} ) \| \, d \boldsymbol{x} - h (\beta\,;\, p_k , \mathcal{X}_k ) \Big)^2$ by its smoothed version. Expansion of the quadratic form yields,

\begin{eqnarray}
& & \Big( \int_{\mathcal{X}_k} \Sh(\angle \nabla f ( e^{s} \boldsymbol{R}_\alpha \boldsymbol{x} + \boldsymbol{b} ) - \alpha - \beta) e^{s} \|  \nabla f ( e^{s} \boldsymbol{R}_\alpha \boldsymbol{x} + \boldsymbol{b} ) \| \, d \boldsymbol{x} - h (\beta\,;\, p_k , \mathcal{X}_k ) \Big)^2 \\
&=& \Big( \int_{\mathcal{X}_k} \Sh( \angle \nabla f ( e^{s} \boldsymbol{R}_\alpha \boldsymbol{x} + \boldsymbol{b} ) - \alpha - \beta) e^{s} \|  \nabla f ( e^{s} \boldsymbol{R}_\alpha \boldsymbol{x} + \boldsymbol{b} ) \| \, d \boldsymbol{x} \Big)^2 + h^2 (\beta\,;\, p_k , \mathcal{X}_k ) \\
& & -2 \Big( \int_{\mathcal{X}_k} \Sh( \angle \nabla f ( e^{s} \boldsymbol{R}_\alpha \boldsymbol{x} + \boldsymbol{b} ) - \alpha - \beta) e^{s} \|  \nabla f ( e^{s} \boldsymbol{R}_\alpha \boldsymbol{x} + \boldsymbol{b} ) \| \, d \boldsymbol{x} \Big) \times  h (\beta\,;\, p_k , \mathcal{X}_k ) \,.
\end{eqnarray}

The first term can be rewritten as below,

\begin{eqnarray}
& & \Big( \int_{\mathcal{X}_k} \Sh( \angle \nabla f ( e^{s} \boldsymbol{R}_\alpha \boldsymbol{x} + \boldsymbol{b} ) - \alpha - \beta) e^{s} \|  \nabla f ( e^{s} \boldsymbol{R}_\alpha \boldsymbol{x} + \boldsymbol{b} ) \| \, d \boldsymbol{x} \Big)^2 \\
&=& e^{2s} \int_{\mathcal{X}_k} \,\, |\{\boldsymbol{x}_2 \,|\, \angle \nabla f ( e^{s} \boldsymbol{R}_\alpha \boldsymbol{x}_2 + \boldsymbol{b} ) = \angle \nabla f ( e^{s} \boldsymbol{R}_\alpha \boldsymbol{x} + \boldsymbol{b} )\}| \\
& & \times \Sh( \angle \nabla f ( e^{s} \boldsymbol{R}_\alpha \boldsymbol{x} + \boldsymbol{b} ) - \alpha - \beta) \, \|  \nabla f ( e^{s} \boldsymbol{R}_\alpha \boldsymbol{x} + \boldsymbol{b} ) \|^2 \, d \boldsymbol{x} \,.
\end{eqnarray}

Note that $|\{\boldsymbol{x}_2 \,|\, \angle \nabla f ( e^{s} \boldsymbol{R}_\alpha \boldsymbol{x}_2 + \boldsymbol{b} ) = \angle \nabla f ( e^{s} \boldsymbol{R}_\alpha \boldsymbol{x} + \boldsymbol{b} )\}| \geq 1$ because at least there is one such $\boldsymbol{x}_2$ for which $\angle \nabla f ( e^{s} \boldsymbol{R}_\alpha \boldsymbol{x}_2 + \boldsymbol{b} ) = \angle \nabla f ( e^{s} \boldsymbol{R}_\alpha \boldsymbol{x} + \boldsymbol{b} )$ holds, that is $\boldsymbol{x}_2 = \boldsymbol{x}$. However, we assume that the cardinality of the set is exactly one, i.e. besides $\boldsymbol{x}_2=\boldsymbol{x}$, there is no other choice for $\boldsymbol{x}_2$ so that the condition $\angle \nabla f ( e^{s} \boldsymbol{R}_\alpha \boldsymbol{x}_2 + \boldsymbol{b} ) = \angle \nabla f ( e^{s} \boldsymbol{R}_\alpha \boldsymbol{x} + \boldsymbol{b} )$ can hold. The rationale is that the variables are continuous and thus their representation has infinite precision. The odds that the gradient orientation at two different points in the image are \emc{exactly} the same is almost impossible, although they might be very close. With this assumption, the quadratic form simplifies as below,

\begin{eqnarray}
& & \Big( \int_{\mathcal{X}_k} \Sh(\angle \nabla f ( e^{s} \boldsymbol{R}_\alpha \boldsymbol{x} + \boldsymbol{b} ) - \alpha - \beta) e^{s} \|  \nabla f ( e^{s} \boldsymbol{R}_\alpha \boldsymbol{x} + \boldsymbol{b} ) \| \, d \boldsymbol{x} - h (\beta\,;\, p_k , \mathcal{X}_k ) \Big)^2 \\
&=& e^{2s} \int_{\mathcal{X}_k} \, \Sh( \angle \nabla f ( e^{s} \boldsymbol{R}_\alpha \boldsymbol{x} + \boldsymbol{b} ) - \alpha - \beta) \, \|  \nabla f ( e^{s} \boldsymbol{R}_\alpha \boldsymbol{x} + \boldsymbol{b} ) \|^2 \, d \boldsymbol{x} + h^2 (\beta\,;\, p_k , \mathcal{X}_k ) \\
& & -2 e^s \Big( \int_{\mathcal{X}_k} \Sh(\angle \nabla f ( e^{s} \boldsymbol{R}_\alpha \boldsymbol{x} + \boldsymbol{b} ) - \alpha - \beta) \| \nabla f ( e^{s} \boldsymbol{R}_\alpha \boldsymbol{x} + \boldsymbol{b} ) \| \, d \boldsymbol{x} \Big) \times  h (\beta\,;\, p_k , \mathcal{X}_k ) \\
\label{eq:delta_sift1}
&=& e^{2s} \int_{\mathcal{X}_k} \int_{\mathbb{R}^2}\, \Sh( \angle \nabla f ( \boldsymbol{y} ) - \alpha - \beta) \, \|  \nabla f ( \boldsymbol{y} ) \|^2 \, \delta(\boldsymbol{y} - e^{s} \boldsymbol{R}_\alpha \boldsymbol{x} - \boldsymbol{b}) \, d \boldsymbol{y} \, d \boldsymbol{x} + h^2 (\beta\,;\, p_k , \mathcal{X}_k ) \\
\label{eq:delta_sift2}
& & -2 e^s \Big( \int_{\mathcal{X}_k} \int_{\mathbb{R}^2} \Sh( \angle \nabla f (\boldsymbol{y} ) - \alpha - \beta) \| \nabla f ( \boldsymbol{y} ) \| \, \delta(\boldsymbol{y} - e^{s} \boldsymbol{R}_\alpha \boldsymbol{x} - \boldsymbol{b}) \, d \boldsymbol{y} \, d \boldsymbol{x} \Big) \times h (\beta\,;\, p_k , \mathcal{X}_k ) \,,
\end{eqnarray}

where (\ref{eq:delta_sift1}) and (\ref{eq:delta_sift2}) use the sifting property of the delta function. The goal is to convolve $cost$ with a multivariate Gaussian kernel of covariance $\sigma^2 \boldsymbol{I}$ in variables jointly in $(\alpha,\boldsymbol{b})$. Due to the diagonal form of the covariance, the convolution can be decoupled to that of $\alpha$ and $\boldsymbol{b}$.

We first proceed with smoothing w.r.t. $\boldsymbol{b}$. By linearity of the convolution operator and that the Gaussian kernel integrates to one, we obtained the following,

\begin{eqnarray}
& & [cost(\boldsymbol{c},\alpha,s,\boldsymbol{\,.\,}) \star k_\sigma] (\boldsymbol{b}) \\
&=& q(\boldsymbol{c}) \,+\, \sum_k c_k \, \int_0^{2\pi} [\Big( \int_{\mathcal{X}_k} \Sh( \angle [\nabla f] ( e^{s} \boldsymbol{R}_\alpha \boldsymbol{x} + \,.\, ) - \alpha - \beta) e^{s} \|  [\nabla f] ( e^{s} \boldsymbol{R}_\alpha \boldsymbol{x} + \,.\, ) \| \, d \boldsymbol{x} - h (\beta\,;\, p_k , \mathcal{X}_k ) \Big)^2 \star k_\sigma](\boldsymbol{b}) \, d \beta \\
&=& q(\boldsymbol{c}) \,+\, \sum_k c_k \, \int_0^{2\pi} \Bigg( \\
& & e^{2s} \int_{\mathcal{X}_k} \int_{\mathbb{R}^2}\, \Sh( \angle \nabla f ( \boldsymbol{y} ) - \alpha - \beta) \, \|  \nabla f ( \boldsymbol{y} ) \|^2 \, k_\sigma(\boldsymbol{y} - e^{s} \boldsymbol{R}_\alpha \boldsymbol{x} - \boldsymbol{b}) \, d \boldsymbol{y} \, d \boldsymbol{x} + h^2 (\beta\,;\, p_k , \mathcal{X}_k ) \\
& & -2  e^s \Big( \int_{\mathcal{X}_k} \int_{\mathbb{R}^2} \Sh( \angle \nabla f (\boldsymbol{y} ) - \alpha - \beta) \| \nabla f ( \boldsymbol{y} ) \| \, k_\sigma (\boldsymbol{y} - e^{s} \boldsymbol{R}_\alpha \boldsymbol{x} - \boldsymbol{b}) \, d \boldsymbol{y} \, d \boldsymbol{x} \Big) \times h (\beta\,;\, p_k , \mathcal{X}_k ) \Bigg) \, d \beta \,.
\end{eqnarray}

We now continue by trying to smooth w.r.t. $\alpha$.

\begin{eqnarray}
& & [\,\,[\,\,cost(\boldsymbol{c},\,.\,,s,\boldsymbol{\,.\,}) \star k_\sigma\,\,]\,\, (\boldsymbol{b}) \, \star k_{\tilde{\sigma}}\,\,]\,\, (\alpha)\\
&=& q(\boldsymbol{c}) \,+\, \sum_k c_k \, \int_0^{2\pi} \Bigg( \\
& & e^{2s} \int_{\mathcal{X}_k} \, \int_{\mathbb{R}^2}\, \| \nabla f ( \boldsymbol{y} ) \|^2 \, \big( [ \Sh( \angle \nabla f ( \boldsymbol{y} ) - \,.\, - \beta)  k_\sigma(\boldsymbol{y} - e^{s} \boldsymbol{R}_{\,.\,} \boldsymbol{x} - \boldsymbol{b}) \star k_{\tilde{\sigma}}] (\alpha) \big) \, d \boldsymbol{y} \, d \boldsymbol{x} + h^2 (\beta\,;\, p_k , \mathcal{X}_k ) \\
& & -2  e^s \Big( \int_{\mathcal{X}_k} \int_{\mathbb{R}^2} \, \| \nabla f ( \boldsymbol{y} ) \| \, \big( [ \Sh( \angle \nabla f (\boldsymbol{y} ) - \,.\, - \beta) k_\sigma (\boldsymbol{y} - e^{s} \boldsymbol{R}_{\,.\,} \boldsymbol{x} - \boldsymbol{b}) \star k_{\tilde{\sigma}}] (\alpha) \big) \, d \boldsymbol{y} \, d \boldsymbol{x} \Big) \times h (\beta\,;\, p_k , \mathcal{X}_k ) \Bigg) \, d \beta \,.
\end{eqnarray}

Computation of the convolution $\Sh( \angle \nabla f (\boldsymbol{y} ) - \,.\, + \beta) k_\sigma (\boldsymbol{y} - e^{s} \boldsymbol{R}_{\,.\,} \boldsymbol{x} - \boldsymbol{b}) \star k_{\tilde{\sigma}}$ is \emc{intractable}. However, it can be approximated by applying the convolution only to the $\Sh$ function. The rationale behind this approximation is that Gaussian convolution affects delta function much more than the Gaussian factor\footnote{Gaussian smoothing affects high frequency functions more than low frequency ones; essentially it kills high frequency components, while leaving low frequency components intact.}.

\begin{eqnarray}
& & [\Sh( \angle \nabla f (\boldsymbol{y} ) - \,.\, - \beta) k_\sigma (\boldsymbol{y} - e^{s} \boldsymbol{R}_{\,.\,} \boldsymbol{x} - \boldsymbol{b}) \star k_{\tilde{\sigma}}] (\alpha) \\
&\approx& \Big( [\Sh( \angle \nabla f (\boldsymbol{y} ) - \,.\, - \beta) \star k_{\tilde{\sigma}} ] (\alpha) \Big) k_\sigma (\boldsymbol{y} - e^{s} \boldsymbol{R}_{\alpha} \boldsymbol{x} - \boldsymbol{b}) \\
&=& \tilde{k}_{\tilde{\sigma}}( \angle \nabla f (\boldsymbol{y} ) - \alpha - \beta) k_\sigma (\boldsymbol{y} - e^{s} \boldsymbol{R}_{\alpha} \boldsymbol{x} - \boldsymbol{b}) \,.
\end{eqnarray}

Using this approximation, it follows that,

\begin{eqnarray}
& & [\,\,[\,\,cost(\boldsymbol{c},\,.\,,s,\boldsymbol{\,.\,}) \star k_\sigma\,\,]\,\, (\boldsymbol{b}) \, \star k_{\tilde{\sigma}}\,\,]\,\, (\alpha)\\
&\approx& q(\boldsymbol{c}) \,+\, \sum_k c_k \, \int_0^{2\pi} \Bigg( \\
& & e^{2s} \int_{\mathcal{X}_k} \int_{\mathbb{R}^2}\, \tilde{k}_{\tilde{\sigma}} ( \angle \nabla f ( \boldsymbol{y} ) - \alpha - \beta) \, \|  \nabla f ( \boldsymbol{y} ) \|^2 \, k_\sigma(\boldsymbol{y} - e^{s} \boldsymbol{R}_\alpha \boldsymbol{x} - \boldsymbol{b}) \, d \boldsymbol{y} \, d \boldsymbol{x} + h^2 (\beta\,;\, p_k , \mathcal{X}_k ) \\
& & -2 e^s \Big( \int_{\mathcal{X}_k} \int_{\mathbb{R}^2} \tilde{k}_{\tilde{\sigma}} ( \angle \nabla f (\boldsymbol{y} ) - \alpha - \beta) \| \nabla f ( \boldsymbol{y} ) \| \, k_\sigma (\boldsymbol{y} - e^{s} \boldsymbol{R}_\alpha \boldsymbol{x} - \boldsymbol{b}) \, d \boldsymbol{y} \, d \boldsymbol{x} \Big) \times h (\beta\,;\, p_k , \mathcal{X}_k ) \Bigg) \, d \beta \,.
\end{eqnarray}

\pagebreak
\section{Proof of Lemma 2}

1. revert to previous proof with Z, just for the propsotiion. Do poper variable replacement in the proof.

Use proposition. Justify why Df having no zero component makes sense.

Mention integration w.r.t y is over $X \cap X_k$. We assume that this integral is zero outside of the domain $R^2-X$.

\footnote{These conditions can be assumed as granted. The gradient is no where perfectly zero in the image. It is perfectly zero outside of the image $f$, but that can be taken care of by limiting the integration domain of $\boldsymbol{y}$ from $\mathbb{R}^2$ to $\mathcal{X}$. Having $x_1=0$ or $x_2=0$ has zero measure, and can be removed from the integration w.r.t. $\boldsymbol{x}$ without affecting the integration result.}

\begin{proposition}
\begin{eqnarray}
& & \delta(r\, \boldsymbol{v}(\beta) - \boldsymbol{A}^T \nabla f ( \boldsymbol{y} ) ) \,  k_\sigma (\boldsymbol{y} - \boldsymbol{A} \boldsymbol{x} - \boldsymbol{b}) \\
&=& k_{\sigma} (\frac{\boldsymbol{x}^T \boldsymbol{z} - (\boldsymbol{y}-\boldsymbol{b})^T \nabla f(\boldsymbol{y})}{\| \nabla f(\boldsymbol{y})\|}) \, k_{\sqrt{\sigma^2 + {\sigma^\dag}^2 \, \| \boldsymbol{x} \|^2}}(\frac{(\nabla f (\boldsymbol{y}))^T (\boldsymbol{A x} + \boldsymbol{b} - \boldsymbol{y})^\perp}{ \| \nabla f (\boldsymbol{y})\| }) \, k_{\sigma^\dag \, \| \nabla f(\boldsymbol{y}) \|}(\boldsymbol{z} - \boldsymbol{A}^T \nabla f (\boldsymbol{y}) ) \,.
\end{eqnarray}

\end{proposition}

\linesep

\begin{eqnarray}
\nabla f (\boldsymbol{y}) &\leftrightarrow& \boldsymbol{g}\\
- \boldsymbol{z} &\leftrightarrow& \boldsymbol{c}\\
\boldsymbol{b} - \boldsymbol{y} &\leftrightarrow& \boldsymbol{d}\\
\end{eqnarray}

\begin{proposition}
\begin{eqnarray}
& & [ \big( \delta({\,\boldsymbol{.}\,}^T \boldsymbol{g} + \boldsymbol{c}) \,\,  k_\sigma (\,\boldsymbol{.}\, \boldsymbol{x} + \boldsymbol{d}) \big) \star k_{\sigma^\dag}] (\boldsymbol{A})\\
&=& k_{\sigma} (\frac{ \boldsymbol{g}^T \boldsymbol{d} - \boldsymbol{x}^T \boldsymbol{c} }{\| \boldsymbol{g}\|}) \, k_{\sqrt{\sigma^2 + {\sigma^\dag}^2 \, \| \boldsymbol{x} \|^2}}(\frac{ \boldsymbol{g}^T (\boldsymbol{A x} + \boldsymbol{d})^\perp}{ \| \boldsymbol{g}\| }) \, k_{\sigma^\dag \, \| \boldsymbol{g} \|}(\boldsymbol{A}^T \boldsymbol{g} + \boldsymbol{c}) \,.
\end{eqnarray}

\end{proposition}

\begin{proof}

We first provide an outline of the proof. The $\delta$ function can be replaced by the limit of a Gaussian whose variance tends to zero $\lim_{\epsilon \rightarrow 0} k_\epsilon$. Now $k_\epsilon$ and $k_{\sigma^\dag}$ form the product of Gaussians. The idea is to write this product as a new single Gaussian in $\boldsymbol{A}$ (because then we know how to convolve two Gaussians). We do this by replacing the pair with a single exponential whose exponent is trivially the sum of the original exponents. Using \emc{completing the square} method for the joint exponent, the center and covariance of the single Gaussian emerges. There is a problem though; the resulted quadratic form will have a \emc{singular} covariance whose inverse does not exist\footnote{The inverse of the covariance is required as it directly appears in the definition of the Gaussian.}. We tackle this problem by the \emc{change of coordinate} system.

We begin with the coordinate system transform. Since the covariance of the Gaussian kernel is isotropic, the resulted Gaussian is \emc{radially symmetric}, i.e. $k_\sigma(\boldsymbol{x}) = k_\sigma(\boldsymbol{R} \boldsymbol{x})$ for any rotation matrix $\boldsymbol{R}$. Consequently, instead of directly smoothing the above expression, we can rotate the coordinate system, smooth in the latter system, and then invert the rotation to obtain the smoothed function in the original coordinate. In particular, we use the following rotation matrix,

\begin{equation}
\boldsymbol{R} \triangleq \frac{1}{\| \boldsymbol{x}\| \,\, \|\boldsymbol{g} \|} \left[
\begin{array}{cccc}
g_2 x_2 \sign \big( g_1 x_1 \big) & -g_2 \,|x_1|\, \sign \big(g_1\big)& -
  x_2 \,|g_1|\, \sign(x_1) & \,|g_1 x_1| \\
-g_1 x_2 \sign \big( g_2 x_1 \big)  & 
 g_1 \,|x_1|\, \sign \big( g_2 \big) & -x_2 \,|g_2|\, \sign(x_1) &  
 \,|g_2 x_1|\\
-g_2 x_1 \sign \big( g_1 x_2 \big)  & -g_2 \,|x_2|\, \sign \big( g_1 \big)& x_1 \,|g_1|\, \sign(x_2) & \,|g_1 x_2| \\
g_1 x_1 \sign \big( g_2 x_2\big) & 
 g_1 \,|x_2|\, \sign \big( g_2 \big)& x_1 \,|g_2|\, \sign(x_2) & \,|g_2 x_2| 
\end{array}
\right] \,.
\end{equation}

Due to the assumptions $g_1 \neq 0$, $g_2 \neq 0$, $x_1 \neq 0$, and $x_2 \neq 0$, $\boldsymbol{R}$ is well-defined. Let $\boldsymbol{a} \triangleq \vecc(\boldsymbol{A})$, i.e. $\boldsymbol{a}=(a_{11},a_{12},a_{21},a_{22})$, and let $\boldsymbol{U} \triangleq \boldsymbol{R} \boldsymbol{A}$ and $\boldsymbol{u} \triangleq \vecc(\boldsymbol{U})$. Changing the coordinate system from $\boldsymbol{A}$ to $\boldsymbol{U}$ leads to the following identity,

\begin{eqnarray}
& & k_\epsilon ({\,\boldsymbol{A}\,}^T \boldsymbol{g} + \boldsymbol{c}) \,\,  k_\sigma (\,\boldsymbol{A}\, \boldsymbol{x} + \boldsymbol{d}) \\
&=& k_\epsilon ({\,(\boldsymbol{R}^T \boldsymbol{U})\,}^T \boldsymbol{g} + \boldsymbol{c}) \,\,  k_\sigma (\,\boldsymbol{R}^T \boldsymbol{U}\, \boldsymbol{x} + \boldsymbol{d}) \\
&=& \frac{1}{\sqrt{2 \pi}} e^{\frac {-\epsilon^2 \|\boldsymbol{x}\|^2 (\boldsymbol{d}^T \boldsymbol{g}^\perp)^2 - 
 \sigma^2 \| \boldsymbol{g} \|^2 (\| \boldsymbol{g} \|^2 \| \boldsymbol{d} \|^2 + (2 (-\boldsymbol{d})^T \boldsymbol{g} + \boldsymbol{x}^T \boldsymbol{c}) (\boldsymbol{x}^T \boldsymbol{c}))}{ 2 \sigma^2 \| \boldsymbol{g} \|^2 (\sigma^2 \| \boldsymbol{g} \|^2 + \epsilon^2 \|\boldsymbol{x}\|^2)}}\\
& & \times e^{ \frac {( g_1 (-d_2) - g_2 (-d_1))^2 }{ 2 \sigma^2 \| \boldsymbol{g} \|^2 }} \\
& & \times \frac{1}{\| \boldsymbol{g} \|} \, k_{\frac{\epsilon}{\| \boldsymbol{g} \|}}(u_2 - \frac { - x_1 c_2 + x_2 c_1} {\| \boldsymbol{g} \| \, \| \boldsymbol{x} \| \sign(g_2 x_1)}) \\
& & \times \frac{1}{ \| \boldsymbol{x} \|} k_{\frac{\sigma}{ \| \boldsymbol{x} \|}}(u_3 - \frac {g_1 (-d_2) - g_2 (-d_1)} {\| \boldsymbol{g} \| \, \| \boldsymbol{x} \| \sign(g_1 x_2)}) \\
& & \times \frac{1}{\sqrt{\sigma^2 \| \boldsymbol{g} \|^2 \, + \epsilon^2 \| \boldsymbol{x}\|^2}} \, k_{\frac{\sigma \epsilon}{\sqrt{\sigma^2 \| \boldsymbol{g} \|^2 \, + \epsilon^2 \| \boldsymbol{x}\|^2}}}(u_4 - \frac {\epsilon^2 \|\boldsymbol{x}\|^2 (-\boldsymbol{d})^T \boldsymbol{g}  - \sigma^2 \| \boldsymbol{g} \|^2 \, \boldsymbol{x}^T \boldsymbol{c}} {\| \boldsymbol{g} \| \, \| \boldsymbol{x} \| \, (\sigma^2 \| \boldsymbol{g} \|^2 \, + \epsilon^2 \| \boldsymbol{x}\|^2) \sign(g_2 x_2)}) \,.
\end{eqnarray}

The value of the coordinate transformation is that we can now write this expression as the product of independent Gaussian kernels. Convolution of this expression with the isotropic kernel $k_{\sigma^\dag}(\boldsymbol{u})$ is straightforward,

\begin{eqnarray}
& & \delta(\boldsymbol{z} - \boldsymbol{A}^T \nabla f ( \boldsymbol{y} ) ) \,  k_\sigma (\boldsymbol{y} - \boldsymbol{A} \boldsymbol{x} - \boldsymbol{b}) \\
&\leftrightarrow& \frac{1}{\sqrt{2 \pi}} e^{\frac {\epsilon^2 \|\boldsymbol{x}\|^2 ( ((\boldsymbol{y}-\boldsymbol{b})^T \nabla f (\boldsymbol{y}))^2 - \| \nabla f(\boldsymbol{y}) \|^2 \| \boldsymbol{y} - \boldsymbol{b} \|^2) - 
 \sigma^2 \| \nabla f(\boldsymbol{y}) \|^2 (\| \nabla f(\boldsymbol{y}) \|^2 \| \boldsymbol{y} - \boldsymbol{b} \|^2 - (2 (\boldsymbol{y}-\boldsymbol{b})^T \nabla f (\boldsymbol{y}) - \boldsymbol{x}^T \boldsymbol{z}) (\boldsymbol{x}^T \boldsymbol{z}))}{ 2 \sigma^2 \| \nabla f(\boldsymbol{y}) \|^2 (\sigma^2 \| \nabla f(\boldsymbol{y}) \|^2 + \epsilon^2 \|\boldsymbol{x}\|^2)}}\\
& & \times e^{ \frac {( f_1(\boldsymbol{y}) (y_2-b_2) - f_2(\boldsymbol{y}) (y_1-b_1))^2 }{ 2 \sigma^2 \| \nabla f(\boldsymbol{y}) \|^2 }} \\
& & \times \frac{1}{\| \nabla f(\boldsymbol{y}) \|} \, k_{\sqrt{{\sigma^\dag}^2 + \frac{\epsilon^2}{\| \nabla f(\boldsymbol{y}) \|^2}}}(u_2 - \frac { x_1 z_2 - x_2 z_1} {\| \nabla f(\boldsymbol{y}) \| \, \| \boldsymbol{x} \| \sign(f_2(\boldsymbol{y}) x_1)}) \\
& & \times \frac{1}{ \| \boldsymbol{x} \|} k_{\sqrt{{\sigma^\dag}^2 + \frac{\sigma^2}{\| \boldsymbol{x} \|^2}}}(u_3 - \frac {f_1(\boldsymbol{y}) (y_2-b_2) - f_2(\boldsymbol{y}) (y_1-b_1)} {\| \nabla f(\boldsymbol{y}) \| \, \| \boldsymbol{x} \| \sign(f_1(\boldsymbol{y}) x_2)}) \\
& & \times \frac{1}{\sqrt{\sigma^2 \| \nabla f(\boldsymbol{y}) \|^2 \, + \epsilon^2 \| \boldsymbol{x}\|^2}} \, k_{\sqrt{{\sigma^\dag}^2 + \frac{\sigma^2 \epsilon^2}{\sigma^2 \| \nabla f(\boldsymbol{y}) \|^2 \, + \epsilon^2 \| \boldsymbol{x}\|^2}}} (u_4 - \frac {\epsilon^2 \|\boldsymbol{x}\|^2 (\boldsymbol{y}-\boldsymbol{b})^T \nabla f(\boldsymbol{y})  + \sigma^2 \| \nabla f (\boldsymbol{y}) \|^2 \, \boldsymbol{x}^T \boldsymbol{z}} {\| \nabla f(\boldsymbol{y}) \| \, \| \boldsymbol{x} \| \, (\sigma^2 \| \nabla f(\boldsymbol{y}) \|^2 \, + \epsilon^2 \| \boldsymbol{x}\|^2) \sign(f_2(\boldsymbol{y}) x_2)}) \,.
\end{eqnarray}

Setting $\epsilon \rightarrow 0$, and given that $\sigma>0$ and $\| \nabla f(\boldsymbol{y}) \| \neq 0$, it follows thats,

\begin{eqnarray}
& & \delta(\boldsymbol{z} - \boldsymbol{A}^T \nabla f ( \boldsymbol{y} ) ) \,  k_\sigma (\boldsymbol{y} - \boldsymbol{A} \boldsymbol{x} - \boldsymbol{b}) \\
&\leftrightarrow& \frac{1}{\sqrt{2 \pi}} e^{ 
-\frac{( (\boldsymbol{b}-\boldsymbol{y})^T \nabla f(\boldsymbol{y}) + \boldsymbol{x}^T \boldsymbol{z})^2}{ 2 \sigma^2 \| \nabla f(\boldsymbol{y}) \|^2}}\\
& & \times \frac{1}{\| \nabla f(\boldsymbol{y}) \|} \, k_{\sigma^\dag}(u_2 - \frac { x_1 z_2 - x_2 z_1} {\| \nabla f(\boldsymbol{y}) \| \, \| \boldsymbol{x} \| \sign(f_2(\boldsymbol{y}) x_1)}) \\
& & \times \frac{1}{ \| \boldsymbol{x} \|} k_{\sqrt{{\sigma^\dag}^2 + \frac{\sigma^2}{\| \boldsymbol{x} \|^2}}}(u_3 - \frac {f_1(\boldsymbol{y}) (y_2-b_2) - f_2(\boldsymbol{y}) (y_1-b_1)} {\| \nabla f(\boldsymbol{y}) \| \, \| \boldsymbol{x} \| \sign(f_1(\boldsymbol{y}) x_2)}) \\
& & \times \frac{1}{\sigma \| \nabla f(\boldsymbol{y}) \|} \, k_{\sigma^\dag} (u_4 - \frac {\boldsymbol{x}^T \boldsymbol{z}} {\| \nabla f(\boldsymbol{y}) \| \, \| \boldsymbol{x} \| \, \sign(f_2(\boldsymbol{y}) x_2)}) \,.
\end{eqnarray}

By inverting the coordinate system from $\boldsymbol{u}$ to $(a_{11},a_{12},a_{21},a_{22})$ we obtain,

\begin{equation}
k_{\sigma} (\frac{\boldsymbol{x}^T \boldsymbol{z} - (\boldsymbol{y}-\boldsymbol{b})^T \nabla f(\boldsymbol{y})}{\| \nabla f(\boldsymbol{y})\|}) \, k_{\sqrt{\sigma^2 + {\sigma^\dag}^2 \, \| \boldsymbol{x} \|^2}}(\frac{(\nabla f (\boldsymbol{y}))^T (\boldsymbol{A x} + \boldsymbol{b} - \boldsymbol{y})^\perp}{ \| \nabla f (\boldsymbol{y})\| }) \, k_{\sigma^\dag \, \| \nabla f(\boldsymbol{y}) \|}(\boldsymbol{z} - \boldsymbol{A}^T \nabla f (\boldsymbol{y}) ) \,.
\end{equation}

As a sanity check, we can see that the above expression becomes the same as the original non-smoothed function when $\sigma^\dag \rightarrow 0$,

\begin{eqnarray}
& & \lim_{\sigma^\dag \rightarrow 0} k_{\sigma} (\frac{\boldsymbol{x}^T \boldsymbol{z} - (\boldsymbol{y}-\boldsymbol{b})^T \nabla f(\boldsymbol{y})}{\| \nabla f(\boldsymbol{y})\|}) \, k_{\sqrt{\sigma^2 + {\sigma^\dag}^2 \, \| \boldsymbol{x} \|^2}}(\frac{(\nabla f (\boldsymbol{y}))^T (\boldsymbol{A x} + \boldsymbol{b} - \boldsymbol{y})^\perp}{ \| \nabla f (\boldsymbol{y})\| }) \, k_{\sigma^\dag \, \| \nabla f(\boldsymbol{y}) \|}(\boldsymbol{z} - \boldsymbol{A}^T \nabla f (\boldsymbol{y}) ) \\
&=& k_{\sigma} (\frac{\boldsymbol{x}^T \boldsymbol{z} - (\boldsymbol{y}-\boldsymbol{b})^T \nabla f(\boldsymbol{y})}{\| \nabla f(\boldsymbol{y})\|}) \, k_{\sigma}(\frac{(\nabla f (\boldsymbol{y}))^T (\boldsymbol{A x} + \boldsymbol{b} - \boldsymbol{y})^\perp}{ \| \nabla f (\boldsymbol{y})\| }) \, \delta(\boldsymbol{z} - \boldsymbol{A}^T \nabla f (\boldsymbol{y}) ) \\
&=& k_{\sigma} (\frac{\boldsymbol{x}^T (\boldsymbol{A}^T \nabla f (\boldsymbol{y})) - (\boldsymbol{y}-\boldsymbol{b})^T \nabla f(\boldsymbol{y})}{\| \nabla f(\boldsymbol{y})\|}) \, k_{\sigma}(\frac{(\nabla f (\boldsymbol{y}))^T (\boldsymbol{A x} + \boldsymbol{b} - \boldsymbol{y})^\perp}{ \| \nabla f (\boldsymbol{y})\| }) \, \delta(\boldsymbol{z} - \boldsymbol{A}^T \nabla f (\boldsymbol{y}) ) \\
&=& k_{\sigma,1} (\frac{(\nabla f (\boldsymbol{y}))^T ( \boldsymbol{A x} +\boldsymbol{b} - \boldsymbol{y} )  }{\| \nabla f(\boldsymbol{y})\|}) \, k_{\sigma,1}(\frac{(\nabla f (\boldsymbol{y}))^T (\boldsymbol{A x} + \boldsymbol{b} - \boldsymbol{y})^\perp}{ \| \nabla f (\boldsymbol{y})\| }) \, \delta(\boldsymbol{z} - \boldsymbol{A}^T \nabla f (\boldsymbol{y}) ) \\
&=& k_{\sigma,2} \Big( \frac{1}{\| \nabla f(\boldsymbol{y})\|} \big((\nabla f (\boldsymbol{y}))^T ( \boldsymbol{A x} +\boldsymbol{b} - \boldsymbol{y} ) \,\, , \,\, (\nabla f (\boldsymbol{y}))^T (\boldsymbol{A x} + \boldsymbol{b} - \boldsymbol{y})^\perp \big) \Big) \, \delta(\boldsymbol{z} - \boldsymbol{A}^T \nabla f (\boldsymbol{y}) ) \\
&=& k_{\sigma,2} \Big( \frac{1}{\| \nabla f(\boldsymbol{y})\|} \| \nabla f (\boldsymbol{y}) \|\,\, (\boldsymbol{A x} +\boldsymbol{b} - \boldsymbol{y} ) \Big) \, \delta(\boldsymbol{z} - \boldsymbol{A}^T \nabla f (\boldsymbol{y}) ) \\
&=& k_{\sigma} (\boldsymbol{A x} +\boldsymbol{b} - \boldsymbol{y} ) \, \delta(\boldsymbol{z} - \boldsymbol{A}^T \nabla f (\boldsymbol{y}) ) \\
\end{eqnarray}

\qed

\end{proof}

\label{sec:lemma_2}

The goal is to convolve $d(f \circ \tau_{(\boldsymbol{A} , \boldsymbol{b})} ,p_k, \mathcal{X}_k)$ with the Gaussian kernel. By linearity of the convolution operator we obtain,

\begin{eqnarray}
& & \big[ \big( [ d(f \circ \tau_{(\,.\, , \,.\,)} ,p_k, \mathcal{X}_k) \star k_{\sigma_b}](\boldsymbol{b}) \big) \star k_{\sigma_a} \big](\boldsymbol{A}) \\
&\triangleq&\big[ \big( [ - \int_{\mathcal{X}_k} \int_0^{2\pi} h (\beta , \boldsymbol{x} \,;\, f \circ \tau_{(\,.\,,\,.\,)}) \times h (\beta , \boldsymbol{x} \,;\, p_k) \, d \beta \, d \boldsymbol{x} \star k_{\sigma_b}](\boldsymbol{b}) \big) \star k_{\sigma_a} \big](\boldsymbol{A}) \, d \boldsymbol{x}\\
&=& - \int_{\mathcal{X}_k} \int_0^{2\pi} \Big( \big[ \big( [ h (\beta , \boldsymbol{x} \,;\, f \circ \tau_{(\,.\,,\,.\,)})  \star k_{\sigma_b}](\boldsymbol{b}) \big) \star k_{\sigma_a} \big](\boldsymbol{A}) \Big) \times h (\beta , \boldsymbol{x} \,;\, p_k) \, d \beta \, d \boldsymbol{x} \,.
\end{eqnarray}

Thus in the following we focus on $h (\beta , \boldsymbol{x} \,;\, f \circ \tau_{(\,.\,,\,.\,)})  \star k_{\sigma_b} \star k_{\sigma_a}$. We first manipulate $h (\beta , \boldsymbol{x} \,;\, f \circ \tau_{(\boldsymbol{A} ,\boldsymbol{b} )})$ by applying the chain rule of derivate $
\nabla \big ( f( \boldsymbol{Ax}+\boldsymbol{b}) \big) = \boldsymbol{A}^T \big( [\nabla f] ( \boldsymbol{Ax}+\boldsymbol{b} ) \big)$ followed by the sifting property of the delta function,

\begin{eqnarray}
& & h (\beta , \boldsymbol{x} \,;\, f \circ \tau_{(\boldsymbol{A},\boldsymbol{b})}) \\
&\triangleq& \Sh( \beta - \angle \nabla (f ( \boldsymbol{A} \boldsymbol{x} + \boldsymbol{b} )) ) \| \nabla f ( \boldsymbol{A} \boldsymbol{x} + \boldsymbol{b} ) \|\\
&=& \Sh( \beta - \angle \boldsymbol{A}^T [\nabla f] ( \boldsymbol{A} \boldsymbol{x} + \boldsymbol{b} ) )  \|  \boldsymbol{A}^T \nabla f ( \boldsymbol{A} \boldsymbol{x} + \boldsymbol{b} ) \|\\
&=& \int_{\mathbb{R}^2} \Sh( \beta - \angle \boldsymbol{A}^T \nabla f ( \boldsymbol{y} ) ) \| \boldsymbol{A}^T \nabla f ( \boldsymbol{y} ) \| \, \delta(\boldsymbol{y} - \boldsymbol{A} \boldsymbol{x} - \boldsymbol{b}) \, d \boldsymbol{y}\,.
\end{eqnarray}

Computing the inner convolution, i.e. w.r.t. $\boldsymbol{b}$, is straightforward,

\begin{eqnarray}
& & [h (\beta , \boldsymbol{x} \,;\, f \circ \tau_{(\boldsymbol{A},\,.\,)}) \star k_{\sigma_b}] (\boldsymbol{b})\\
&=& [\Big( \int_{\mathbb{R}^2} \Sh( \beta - \angle \boldsymbol{A}^T \nabla f ( \boldsymbol{y} ) ) \| \boldsymbol{A}^T \nabla f ( \boldsymbol{y} ) \| \, \delta(\boldsymbol{y} - \boldsymbol{A} \boldsymbol{x} - \boldsymbol{b}) \, d \boldsymbol{y} \Big) \star k_{\sigma_b}] (\boldsymbol{b})\\
&=& \int_{\mathbb{R}^2} \Sh( \beta - \angle \boldsymbol{A}^T \nabla f ( \boldsymbol{y} ) ) \| \boldsymbol{A}^T \nabla f ( \boldsymbol{y} ) \| \, k_{\sigma_b} (\boldsymbol{y} - \boldsymbol{A} \boldsymbol{x} - \boldsymbol{b}) \, d \boldsymbol{y}\,.
\end{eqnarray}

The latter can be expressed by the sifting property of the delta function as below,

\begin{eqnarray}
& & [h (\beta , \boldsymbol{x} \,;\, f \circ \tau_{(\boldsymbol{A},\,.\,)}) \star k_{\sigma_b}] (\boldsymbol{b})\\
&=& \int_{\mathbb{R}^2} \int_{\mathbb{R}^2} \delta \Big(\boldsymbol{z} - \boldsymbol{A}^T \nabla f ( \boldsymbol{y} ) \Big) \, \Sh( \beta - \angle \boldsymbol{z} ) \| \boldsymbol{z} \| \, k_{\sigma_b} (\boldsymbol{y} - \boldsymbol{A} \boldsymbol{x} - \boldsymbol{b}) \, d \boldsymbol{z} \, d \boldsymbol{y}\,.
\end{eqnarray}

We now apply a change of variable to move from the Cartesian coordinate $(z_1,z_2)$ to the \emc{polar} coordinate $(r,\phi)$ such that $(z_1,z_2)=(r \cos(\phi),r \sin(\phi))$. This results in replacing $\int_{\mathbb{R}^2} f(z_1,z_2) \, d z_1 \, d z_2$ by $\int_0^\infty \int_0^{2 \pi} r \, f(r \, \boldsymbol{v}(\phi) ) \, d \phi \, d r$, where  $\boldsymbol{v}(\phi) \triangleq (\cos(\phi), \sin(\phi))$. 

\begin{eqnarray}
& & [h (\beta , \boldsymbol{x} \,;\, f \circ \tau_{(\boldsymbol{A},\,.\,)}) \star k_{\sigma_b}] (\boldsymbol{b})\\
&=& \int_{\mathbb{R}^2} \int_0^\infty \int_0^{2 \pi} r \, \delta \Big(r \boldsymbol{v}(\phi) - \boldsymbol{A}^T \nabla f ( \boldsymbol{y} )\Big) \, \Sh( \beta - \angle r \boldsymbol{v}(\phi) ) \| r \boldsymbol{v}(\phi) \| \, k_{\sigma_b} (\boldsymbol{y} - \boldsymbol{A} \boldsymbol{x} - \boldsymbol{b}) \, d \phi \, d r \, d \boldsymbol{y} \\
&=& \int_{\mathbb{R}^2} \int_0^\infty \int_0^{2 \pi} r \, \delta \Big(r \boldsymbol{v}(\phi) - \boldsymbol{A}^T \nabla f ( \boldsymbol{y} ) \Big) \, \Sh( \beta - \angle \boldsymbol{v}(\phi) ) \, r \, \, k_{\sigma_b} (\boldsymbol{y} - \boldsymbol{A} \boldsymbol{x} - \boldsymbol{b}) \, d \phi \, d r \, d \boldsymbol{y} \\
&=& \int_{\mathbb{R}^2} \int_0^\infty r^2 \, \delta \Big(r \boldsymbol{v}(\beta) - \boldsymbol{A}^T \nabla f ( \boldsymbol{y} ) \Big) \, k_{\sigma_b} (\boldsymbol{y} - \boldsymbol{A} \boldsymbol{x} - \boldsymbol{b}) \, d r \, d \boldsymbol{y} \,.
\end{eqnarray}

We are now ready to smooth this form w.r.t. $\boldsymbol{A}$. That is, we want to compute convolution of this expression with a multivariate Gaussian in $(a_{11},a_{12},a_{21},a_{22})$ of covariance $\sigma^2 \boldsymbol{I}$.

\linesep

\linesep

Using this result, we can continue as below,

\begin{eqnarray}
& & [[cost(\boldsymbol{c},\,.\,,\boldsymbol{\,.\,}) \star k_\sigma] (\boldsymbol{b}) \star k_{\sigma^\dag}] (\boldsymbol{A}) \\
&=& q(\boldsymbol{c}) \,+\, \sum_k c_k \, \int_0^{2\pi} \Bigg( \\
& & \int_{\mathcal{X}_k} \int_{\mathbb{R}^2}\, \int_{\mathbb{R}^2} \delta( \boldsymbol{v}^T(\beta) \boldsymbol{z} ) \, \| \boldsymbol{z} \|^2 \, [ \big( \delta(\boldsymbol{z} - \,.\,^T \nabla f ( \boldsymbol{y} ) ) \, k_\sigma(\boldsymbol{y} - \,.\, \boldsymbol{x} - \boldsymbol{b}) \big) \star k_{\sigma^\dag}] (\boldsymbol{A}) \, d \boldsymbol{z} \, d \boldsymbol{y} \, d \boldsymbol{x} + h^2 (\beta\,;\, p_k , \mathcal{X}_k ) \\
& & -2 \Big( \int_{\mathcal{X}_k} \int_{\mathbb{R}^2} \int_{\mathbb{R}^2} \delta( \boldsymbol{v}^T(\beta) \boldsymbol{z} ) \, \| \boldsymbol{z} \| \, [ \big( \delta(\boldsymbol{z} - \,.\,^T \nabla f ( \boldsymbol{y} ) ) \, k_\sigma(\boldsymbol{y} - \,.\, \boldsymbol{x} - \boldsymbol{b}) \big) \star k_{\sigma^\dag}] (\boldsymbol{A}) \, d \boldsymbol{z} \, d \boldsymbol{y} \, d \boldsymbol{x} \Big) \times h (\beta\,;\, p_k , \mathcal{X}_k ) \Bigg) \, d \beta \\
&=& q(\boldsymbol{c}) \,+\, \sum_k c_k \, \int_0^{2\pi} \Bigg( \\
& & \int_{\mathcal{X}_k} \int_{\mathbb{R}^2}\, \int_{\mathbb{R}^2} \delta( \boldsymbol{v}^T(\beta) \boldsymbol{z} ) \, \| \boldsymbol{z} \|^2 \, k_{\sigma} (\frac{\boldsymbol{x}^T \boldsymbol{z} - (\boldsymbol{y}-\boldsymbol{b})^T \nabla f(\boldsymbol{y})}{\| \nabla f(\boldsymbol{y})\|}) \, k_{\sigma^\dag \, \| \nabla f(\boldsymbol{y}) \|}(\boldsymbol{z} - \boldsymbol{A}^T \nabla f (\boldsymbol{y}) ) \, d \boldsymbol{z}\\
& & \times k_{\sqrt{\sigma^2 + {\sigma^\dag}^2 \, \| \boldsymbol{x} \|^2}}(\frac{(\nabla f (\boldsymbol{y}))^T (\boldsymbol{A x} + \boldsymbol{b} - \boldsymbol{y})^\perp}{ \| \nabla f (\boldsymbol{y})\| }) \, d \boldsymbol{y} \, d \boldsymbol{x} + h^2 (\beta\,;\, p_k , \mathcal{X}_k ) \\
& & -2 \Big( \int_{\mathcal{X}_k} \int_{\mathbb{R}^2} \int_{\mathbb{R}^2} \delta( \boldsymbol{v}^T(\beta) \boldsymbol{z} ) \, \| \boldsymbol{z} \| \, k_{\sigma} (\frac{\boldsymbol{x}^T \boldsymbol{z} - (\boldsymbol{y}-\boldsymbol{b})^T \nabla f(\boldsymbol{y})}{\| \nabla f(\boldsymbol{y})\|}) \, k_{\sigma^\dag \, \| \nabla f(\boldsymbol{y}) \|}(\boldsymbol{z} - \boldsymbol{A}^T \nabla f (\boldsymbol{y}) ) \, d \boldsymbol{z}\\
& & \times k_{\sqrt{\sigma^2 + {\sigma^\dag}^2 \, \| \boldsymbol{x} \|^2}}(\frac{(\nabla f (\boldsymbol{y}))^T (\boldsymbol{A x} + \boldsymbol{b} - \boldsymbol{y})^\perp}{ \| \nabla f (\boldsymbol{y})\| })  \, d \boldsymbol{y} \, d \boldsymbol{x} \Big) \times h (\beta\,;\, p_k , \mathcal{X}_k ) \Bigg) \, d \beta \,.
\end{eqnarray}

We now apply a change of variable to move from the Cartesian coordinate $(z_1,z_2)$ to the \emc{polar} coordinate $(r,\phi)$ such that $(z_1,z_2)=(r \cos(\phi),r \sin(\phi))$. This transforms the form $\int_{\mathbb{R}^2} f(z_1,z_2) \, d z_1 \, d z_2$ to $\int_0^\infty \int_0^{2 \pi} r \, f(r \cos(\phi),r \sin(\phi)) \, d \phi \, d r$. 

\begin{eqnarray}
& & \int_{\mathbb{R}^2} \Sh( \beta - \angle \boldsymbol{z} ) \, \| \boldsymbol{z} \| \, k_{\sigma} (\frac{\boldsymbol{x}^T \boldsymbol{z} - (\boldsymbol{y}-\boldsymbol{b})^T \nabla f(\boldsymbol{y})}{\| \nabla f(\boldsymbol{y})\|}) \, k_{\sigma^\dag \, \| \nabla f(\boldsymbol{y}) \|}(\boldsymbol{z} - \boldsymbol{A}^T \nabla f (\boldsymbol{y}) ) \, d \boldsymbol{z} \\
&=& \int_{0}^\infty \int_{0}^{2\pi} r \Sh( \beta - \phi ) \, r \, k_{\sigma} (\frac{r \boldsymbol{x}^T \boldsymbol{v}(\phi) - (\boldsymbol{y}-\boldsymbol{b})^T \nabla f(\boldsymbol{y})}{\| \nabla f(\boldsymbol{y})\|}) \, k_{\sigma^\dag \, \| \nabla f(\boldsymbol{y}) \|}(r \boldsymbol{v}(\phi) - \boldsymbol{A}^T \nabla f (\boldsymbol{y}) ) \, d \phi \, d r\\ 
&=& \int_{0}^\infty \, r^2 \, k_{\sigma} (r \frac{ \boldsymbol{x}^T \boldsymbol{v}(\beta)}{\| \nabla f(\boldsymbol{y})\|} + \frac{(\boldsymbol{b}-\boldsymbol{y})^T \nabla f(\boldsymbol{y})}{\| \nabla f(\boldsymbol{y})\|}  ) \, k_{\sigma^\dag \, \| \nabla f(\boldsymbol{y}) \|}(r \boldsymbol{v}(\beta) - \boldsymbol{A}^T \nabla f (\boldsymbol{y}) ) \, d r\\ 
\label{eq:radial_integral}
&=& \frac{e^{-\frac{((\boldsymbol{b}-\boldsymbol{y})^T \tilde{\nabla} f(\boldsymbol{y}))^2}{2 \sigma^2 }-\frac{\|\boldsymbol{A}^T \tilde{\nabla} f (\boldsymbol{y})\|^2}{2 {\sigma^\dag}^2}}  w(- \frac{ \sigma^{-2} \tilde{\nabla}^T f(\boldsymbol{y}) (\boldsymbol{y}-\boldsymbol{b}) \boldsymbol{x}^T \tilde{\boldsymbol{v}} (\beta, \boldsymbol{y})+{\sigma^\dag}^{-2} \tilde{\nabla}^T f (\boldsymbol{y}) \boldsymbol{A} \tilde{\boldsymbol{v}}(\beta, \boldsymbol{y})  }{2 t})}{8 \sqrt{2} \pi^{\frac{3}{2}} \sigma {\sigma^\dag}^2 \, \| \nabla f(\boldsymbol{y}) \|^2 \, t^3} \,, 
\end{eqnarray}

where $\tilde{\nabla} f (\boldsymbol{y}) \triangleq \frac{\nabla f(\boldsymbol{y})}{\| \nabla f(\boldsymbol{y})\|}$, $\tilde{\boldsymbol{v}}(\beta, \boldsymbol{y})\triangleq \frac{\boldsymbol{v}(\beta)}{\| \nabla f(\boldsymbol{y})\|}$, and $t \triangleq \sqrt{\frac{(\boldsymbol{x}^T \tilde{\boldsymbol{v}}(\beta, \boldsymbol{y}))^2}{2 \sigma^2}+\frac{1}{2 {\sigma^\dag}^2 \, \| \nabla f(\boldsymbol{y}) \|^2}}$ and $w(x) \triangleq \sqrt{\pi} e^{x^2} (1+2 x^2 ) \erfc(x) -2 x $. In (\ref{eq:radial_integral}) we use an elementary identity\footnote{
We use the identity,

\begin{equation}
\int_0^\infty r^2 \, k_{\sigma_1,1}(r c_1 + c_2) 
 k_{\sigma_2,2} (r \boldsymbol{c}_3 + \boldsymbol{c}_4) \, dr = \frac{e^{-\frac{c2^2}{2 \sigma_1^2}-\frac{\|\boldsymbol{c}_4\|^2}{2 \sigma_2^2}}  (\sqrt{\pi} (1+2 t_2^2 ) e^{t_2^2} \erfc(t_2) -2 t_2)}{8 \sqrt{2}  \pi^{\frac{3}{2}} \sigma_1 \sigma_2^2 t_1^3}\,,
\end{equation}

for $t_1 \triangleq \sqrt{\frac{c1^2}{2 \sigma_1^2}+\frac{\|\boldsymbol{c}_3\|^2}{2 \sigma_2^2}}$ and $t_2 \triangleq \frac{\frac{c1 c2}{s1^2}+\frac{\boldsymbol{c}_3^T \boldsymbol{c}_4}{\sigma_2^2}}{2 t_1}$. This identity is derived in two steps:

\begin{enumerate}
\item \emc{Completing the square} of the exponent in the integrand.
\begin{equation}
-\frac{(r c_1 + c_2)^2}{2 \sigma_1^2} - \frac{\|r \boldsymbol{c}_3 + \boldsymbol{c}_4\|^2}{2 \sigma_2^2} \quad=\quad -\frac{1}{2} (r+\frac{\boldsymbol{c}_3^T \boldsymbol{c}_4 \sigma_1^2+c_1 c_2 \sigma_2^2}{\|\boldsymbol{c}_3\|^2 \sigma_1^2+c_1^2 \sigma_2^2})^2  (\frac{c_1^2}{\sigma_1^2}+\frac{\|\boldsymbol{c}_3\|^2}{\sigma_2^2})+\frac{1}{2} ( \frac{(c_1 c_2 \sigma_2^2+\sigma_1^2 \boldsymbol{c}_3^T \boldsymbol{c}_4)^2}{\|\boldsymbol{c}_3\|^2 \sigma_1^4 \sigma_2^2+c_1^2 \sigma_1^2 \sigma_2^4} - \frac{c_2^2}{\sigma_1^2}- \frac{\|\boldsymbol{c}_4\|^2}{\sigma_2^2} )\,.
\end{equation}

\item Using the identity about \emc{Gaussian moments},

\begin{equation}
\int_0^\infty r^2 e^{-\frac{(r - a_1)^2}{2 a_2^2}} \, d r = a_1 a_2^2 e^{-\frac{a_1^2}{2 a_2^2}} + \sqrt{\frac{\pi}{2}} a_2 (a_1^2+a_2^2)  (1+\erf(\frac{a_1}{\sqrt{2} a_2})) \,.
\end{equation}

\end{enumerate}
 
}.

\pagebreak

\end{document}